\documentclass{article}

\usepackage[english]{babel}

\usepackage[letterpaper,top=2cm,bottom=2cm,left=3cm,right=3cm,marginparwidth=1.75cm]{geometry}

\usepackage{amsmath}
\usepackage{amsthm}
\usepackage{amssymb}
\usepackage{graphicx}

\usepackage{algorithm}  
\usepackage{algpseudocode}  
\usepackage{makecell}
\usepackage{booktabs}
\usepackage{multirow}
\usepackage{diagbox}
\usepackage{cite}
\theoremstyle{definition}
\newtheorem{theorem}{Theorem}
\newtheorem{definition}{Definition}
\usepackage[colorlinks=true, allcolors=blue]{hyperref}
\usepackage{caption}
\usepackage{subcaption}
\usepackage{ulem}
\usepackage{indentfirst}
\usepackage{authblk}
\usepackage{enumerate}
\title{Quaternion-based dynamic mode decomposition for background modeling in color videos}
\author[a]{Juan Han\thanks{E-mail: juanhan0604@163.com}}
\author[a]{Kit Ian Kou\thanks{Corresponding author: kikou@umac.mo }}
\author[a]{Jifei Miao\thanks{E-mail: jifmiao@163.com}}

\affil[a]{Department of Mathematics, Faculty of
	Science and Technology, University of Macau, Macau 999078, China}

\date{}

\begin{document}

\maketitle

\begin{abstract}
Scene Background Initialization (SBI) is one of the challenging problems in computer vision. Dynamic mode decomposition (DMD) is a recently proposed method to robustly decompose a video sequence into the background model and the corresponding foreground part. However, this method needs to convert the color image into the grayscale image for processing, which leads to the neglect of the coupling information between the three channels of the color image. In this study, we propose a quaternion-based DMD (Q-DMD), which extends the DMD by quaternion matrix analysis, so as to completely preserve the inherent color structure of the color image and the color video. We exploit the standard eigenvalues of the quaternion matrix to compute its spectral decomposition and calculate the corresponding Q-DMD modes and eigenvalues. The results on the publicly available benchmark datasets prove that our Q-DMD outperforms the exact DMD method, and experiment results also demonstrate that the performance of our approach is comparable to that of the state-of-the-art ones.   
\end{abstract}


\section{Introduction}

\indent
Scene Background Initialization (SBI) remains an important task in the field of scene background modeling, and it has a wide range of applications, such as video surveillance, video segmentation, video compression, video inpainting, privacy protection for videos, computation photography, and so on. The aim of SBI is to extract a background model with no foreground objects from a video sequence whose background is occluded by any number of foreground objects \cite{maddalena2015towards}. In real scenes, there exist challenges to achieve this aim 
such as sudden illumination changes, night videos, low framerate, dynamic background, camera jitter, and so on. In order to address these challenges, various background modeling methods have been proposed over the last few decades. Background modeling approaches based on deep learning have been developed in recent years. Schofield et al. \cite{schofield1996system} were the first to propose a Random Access Memory (RAM) based neural network method to identify sections of the background scene in each test image. However, this approach requires the background of the scene to be correctly represented by the images, and there is no background maintenance stage. Maddalena and Petrosino \cite{maddalena2008self} proposed a method based on self organization through artificial neural networks named Self Organizing Background Subtraction (SOBS) which learns background motion trajectories in a self organizing manner. Convolutional Neural Networks (CNNs) have also been used for background modeling by Braham and Droogenbroeck \cite{braham2016deep}, Bautista et al. \cite{bautista2016convolutional} and Cinelli \cite{cinelli2017anomaly}. Halfaoui et al. \cite{halfaoui2016cnn} proposed a reliable CNN-based approach to obtain the initial background of a scene by using just a small set of frames containing foreground objects. A detailed overview of deep learning based methods, we recommend \cite{bouwmans2019deep}.

In addition, among all the methods for separating background and foreground, one of the representative frameworks is to decompose video sequences into low rank matrices (i.e., background) and sparse matrices (i.e., foreground). Based on this viewpoint, Cand$\grave{e}$s et al. \cite{candes2011robust} proposed the first framework of robust principal component analysis (RPCA). There are many variants of RPCA according to the difference of decomposition, loss function, the optimization problem and the solver used. The recent review \cite{bouwmans2017decomposition} provides a detailed overview of some traditional and state-of-the-art methods. By incorporating spatiotemporal sparse subspace clustering into the framework of RPCA, Javed et al. \cite{javed2017background, javed2018moving} proposed a Motion-assisted Spatiotemporal Clustering of Low-rank (MSCL) approach for both the background estimation and foreground segmentation. Laugraud et al. \cite{laugraud2017labgen} developed a method named LaBGen which combines a pixel-wise temporal median filter and a patch selection mechanism based on motion detection for the generation of the stationary background of the scene. 
 
A data matrix can also be decomposed into a low rank matrix and a sparse matrix by using dynamic mode decomposition (DMD). This method can easily distinguish the static background and dynamic foreground from the data matrix by differentiating between the near-zero temporal Fourier modes and the remaining modes bounded away from the origin \cite{kutz2015multi}. DMD also has many variants, such as standard DMD \cite{brunton2019data}, exact DMD \cite{brunton2019data}, multi-resolution DMD \cite{kutz2015multi} and compressed DMD \cite{erichson2019compressed}. However, when it comes to the application of background modeling of video frames, these variant methods deal with gray-scale video sequences directly. For color video processing, DMD and its variants ignore the mutual connection among red, green and blue channels, because these methods are applied to red, green and blue channels separately, which may cause color distortion when separating the foreground and background of color video frames. 

In this paper, we use a quaternion matrix to represent a color image, and then extend the DMD method to the quaternion system. Many studies have shown that quaternion representation of color images can achieve excellent results in color image processing problems, such as \cite{ell2006hypercomplex, le2004singular, hosny2019new, li2015quaternion} and so on. The quaternion-based method (Q-DMD) is to represent a pixel of a color image with a pure quaternion, that is, the three channels RGB of a color pixel correspond to the three imaginary parts of a quaternion, respectively, which can be shown by the following formula:
\[ \dot p = 0+p_R \, \cdot i+p_G \,  \cdot j +p_B \,  \cdot k,\]
where $\dot p$ represents a color pixel, and $p_R$, $p_G$, $p_B$ correspond to the pixel values of the three channels RGB of this color pixel, and $i$, $j$, $k$ are the three imaginary units of a quaternion. 

The main contribution of this paper is to extend the DMD method to the quaternion system, so as to make full use of the coupling information between the three color channels RGB of a color video sequence when separating the video frames into background model and foreground component. To this end, we need to obtain the spectral decomposition of a quaternion matrix. However, a quaternion matrix has infinite right eigenvalues \cite{zhang1997quaternions}, and its left eigenvalues are difficult to obtain. Therefore, in this paper, we establish the spectral decomposition of a quaternion matrix by using the standard eigenvalues of the quaternion matrix and the corresponding eigenvectors \cite{zhang1997quaternions}.

The outline of the rest of this paper is organized as follows. Section 2 sets some notations and introduces preliminaries for quaternion algebra. Section 3 briefly reviews the theory of the DMD method. In Section 4, we introduce the eigenvalues, eigenvectors, spectral decomposition, and singular value decomposition of a quaternion matrix. Then we propose the quaternion-based DMD (Q-DMD) which can be used to separate a color video sequences into an approximate low-rank structure and a sparse structure. Section 5 provides some experiments on the basis of benchmark datasets to illustrate the performance of our approach, and compare our method with DMD and some state-of-the-art methods. Finally, some conclusions are given in Section 6.

\section{Notations and preliminaries}
In this section, we first summarize some main notations and then give a brief review of some basic knowledge of quaternion algebra. In addition, we recommend \cite{girard2007quaternions} for a more complete introduction of quaternion algebra.
\subsection{Notations}
In this paper, the set of real numbers, the set of complex numbers and the set of quaternions are denoted by $\mathbb{R}$, $\mathbb{C}$, $\mathbb{H}$, respectively. In addition, $C^{+}$ denotes the set of complex numbers with nonnegative imaginary part. Lowercase letters, e.g., $a$, boldface lowercase letters, e.g., $\mathbf{a}$, boldface capital letters, e.g., $\mathbf{A}$ represent scalars, vectors and matrices in real and complex fields, respectively. A quaternion scalar, a quaternion vector and a quaternion matrix are written as $\dot{q}$, $\mathbf{\dot{q}}$, $\mathbf{\dot{Q}}$, respectively. $(\cdot)$, $(\cdot)^{-1}$, $(\cdot)^{\dagger}$, $(\cdot)^{T}$, $(\cdot)^{\ast}$, $(\cdot)^{H}$ represent the conjugation, inverse, pseudoinverse, transpose, conjugation, and conjugate transpose, respectively.

\subsection{Basic knowledge of quaternion algebra}
Quaternion was first proposed by W.R. Hamilton in 1843 \cite{hamilton1844ii}. Let  $\dot{q} \in \mathbb{H}$ be a quaternion, 
\[\dot{q}=q_0+q_1 i +q_2 j+q_3 k,\]
where $q_0$, $q_1$, $q_2$, $q_3 \in \mathbb{R}$, and $i$, $j$, $k$ are imaginary number units which obey the quaternion rules that
\[\begin{cases}
i^2 =j^2=k^2=ijk=-1,\\
ij=-ji=k, jk=-kj=i, ki=-ik=j.
\end{cases}
\]
For every quaternion $\dot{q}=q_0+q_1 i +q_2 j+q_3 k$, it can be uniquely rewritten as $\dot{q}=q_0+q_1 i +(q_2 +q_3 i) j = c_1+c_2 j$, where $c_1$, $c_2 \in \mathbb{C}$. In addition, a quaternion $\dot{q} \in \mathbb{H}$ can be decomposed into a scalar part $S(\dot{q})$ and a vector part $\mathbf{V}(\dot{q})$, that is
\[\dot{q} = S(\dot{q}) + \mathbf{V}(\dot{q})\]
 where $S(\dot{q})= q_0 \in \mathbb{R}$ which is called the real part of $\dot{q}$, and $\mathbf{V}(\dot{q}) = \dot{q} - S(\dot{q}) = q_1 i +q_2 j+q_3 k$ is called the vector part. A quaternion whose real part is equal to 0, i.e., $S(\dot{q}) =0$, is called a pure quaternion. The set of pure quaternions is denoted as $\mathbf{V}(\mathbb{H})$. 
 
For a quaternion $\dot{q}$, its conjugate quaternion is defined as 
$\dot{q}^{\ast}=q_0-q_1 i -q_2 j-q_3 k.$ And the norm of a quaternion $\dot{q}$ is defined as
$\lvert{\dot{q}} \rvert =\sqrt{qq^{\ast}}=\sqrt{q^{\ast}q}=\sqrt{q_0^2+q_1^2+q_2^2+q_3^2}.$
Different from the complex number field, the commutative law is generally not valid in the quaternion system, i.e., $\dot{q}_1 \dot{q}_2 \neq \dot{q}_2 \dot{q}_1$ in general. In addition, a pure quaternion which satisfies that $\lvert{\dot \xi} \rvert=1$ is named a pure unit quaternion.

Now, we introduce the definition of exponential and logarithm of a quaternion. 
Every (non-null) pure quaternion $\dot{\xi}$ can be presented by $\dot \xi = \lvert{\xi}\rvert {\dot \xi}_\text{pu}$, where ${\dot \xi}_\text{pu}$ is a pure unit quaternion (i.e., ${\dot \xi}_\text{pu} \in \mathbf{V}(\mathbb{H})$, and $\lvert{{\dot \xi}_\text{pu}}\rvert=1)$. 

\begin{definition} \label{exppure} (The exponential function of a pure quaternion \cite{ell2014quaternion}) 
Assuming that $\dot{\xi}$ is a (non-null) pure quaternion, then its exponential function exp : $\mathbf{V}(\mathbb{H}) \rightarrow \mathbb{H}$ can be defined by exploiting its power series expansion, which is given by the following formula:
\begin{equation}
\begin{aligned}
e^{\dot{\xi}} &=\sum_{n=0}^{+\infty}\frac{{\dot{\xi}}^n}{n!}\\
              &=\sum_{n=0}^{+\infty}\frac{\lvert{{\dot \xi}\rvert}^n {{\dot \xi}_\text{pu}}^n }{n!}\\
              &=\sum_{n=0}^{+\infty}(-1)^m\frac{{\lvert{\dot \xi}\rvert}^{2m}}{(2m)!}+{\dot \xi}_\text{pu}\sum_{n=0}^{+\infty}(-1)^m\frac{{\lvert{\dot \xi}\rvert}^{2m+1}}{(2m+1)!}\\
              & = \text{cos} \lvert{\dot \xi}\rvert + {\dot \xi}_\text{pu}\text{sin} \lvert{\dot \xi}\rvert,
\end{aligned}
\end{equation}
since the pure unit quaternion ${\dot \xi}_\text{pu}$ satisfies the following formula:
\[{{\dot \xi}_\text{pu}}^n = \begin{cases}
(-1)^m, \qquad \  \  \  \ \text{if} \quad  n=2m,\\
(-1)^m{\dot \xi}_\text{pu}, \qquad \text{if} \quad n=2m+1.
\end{cases}
\]
\end{definition}
The exponential function of a pure quaternion can be easily represented by cosine and sine functions just as in the complex case. The difference is that the axis ${\dot \xi}_\text{pu}$ is a pure unit quaternion, while the argument is the modulus of ${\dot \xi}$. Obviously, the exponential of a pure quaternion is a full quaternion, with real part
$\text{cos} \lvert{\dot \xi}\rvert$ and vector part ${\dot \xi}_\text{pu}\text{sin} \lvert{\dot \xi}\rvert$.
Another important property different from the complex exponential case is that, in general, the product of two exponents of a pure quaternion is not an exponent whose argument is equal to the sum of the original exponential arguments, that is, for two pure quaternion $\dot{p}$, $\dot{q}$ with different pure unit quaternions ${\dot{p}}_{\text{pu}}$ and ${\dot{q}}_{\text{pu}}$, then
\[e^{{\lvert{\dot{p}}\rvert}{\dot{p}}_\text{pu}} e^{{\lvert{\dot{q}}\rvert}{\dot{q}}_\text{pu}} \neq e^{{\lvert{\dot{p}}\rvert}{\dot{p}}_\text{pu}+{\lvert{\dot{q}}\rvert}{\dot{q}}_\text{pu}}. \]

The exponential function of full quaternions can be defined based on Definition \ref{exppure}.
\begin{definition} \label{fullexpo} (The exponential function of a full quaternion \cite{ell2014quaternion})
For a quaternion $\dot{q}$, its exponential function exp : $\mathbb{H} \rightarrow \mathbb{H}$ is given by
\begin{equation}
\begin{aligned}
e^{\dot{q}} &=\sum_{n=0}^{+\infty}\frac{{\dot{q}}^n}{n!}\\
              &=e^{S(\dot q)}e^{\mathbf{V}(\dot{q})}\\
              & = e^{S(\dot q)}(\text{cos} \lvert{\mathbf{V}(\dot{q})}\rvert + \widetilde{\mathbf{V}(\dot{q})}\text{sin} \lvert{\mathbf{V}(\dot{q})}\rvert),
\end{aligned}
\end{equation}
where $\mathbf{V}(\dot{q}) = \lvert{\mathbf{V}(\dot{q})}\rvert \widetilde{\mathbf{V}(\dot{q})} \in \mathbf{V}(\mathbb{H})$, and $\widetilde{\mathbf{V}(\dot{q})}$ is the normalized vector part of $\dot{q}$.
\end{definition}

\begin{definition} \label{logarithm} (The logarithm of a quaternion \cite{ell2014quaternion})
The logarithm of the quaternion $\dot q$ is the inverse of the exponential function. This means that for $\dot p$, $\dot q \in \mathbb{H}$, if 
\[e^{\dot p}={\dot q},\]
then
\begin{equation}
 \dot p= \text{ln} \, \dot q.
\end{equation}
\end{definition}
There also exists an expression for the logarithm of $\dot q=q_0+q_1 i +q_2 j+q_3 k$ in terms of its elements, which can be considered as a generalization of the logarithm of a complex number and is given by the following formula:
\begin{equation}
   \begin{aligned}
       \text{ln} \, \dot q &= \text{ln} \, \lvert{\dot q}\rvert + \text{ln} \, {\dot q}_{pu}
       &=\text{ln} \, \lvert{\dot q}\rvert + \dot{\mu}_{\dot q}\phi_{\dot q},
   \end{aligned} 
\end{equation}
where 
\[\begin{cases}
\dot{\mu}_{\dot q} = \frac{q_1 i +q_2 j+q_3 k}{\sqrt{q_1^2 +q_2^2+q_3^2}},\\
\phi_{\dot q}= \text{arctan}(\frac{\sqrt{q_1^2 +q_2^2+q_3^2}}{q_0}).
\end{cases}\]

Similarly, a quaternion matrix $\mathbf{\dot Q}=(\dot q_{m,n}) \in \mathbf{H}^{M \times N}$ is denoted as $\mathbf{\dot Q} = \mathbf{ Q}_0 +\mathbf{ Q}_1 i +\mathbf{ Q}_2 j+\mathbf{ Q}_3 k$ with $\mathbf{ Q}_t \in  \mathbf{R}^{M \times N} (t=0, \, 1, \, 2, \, 3)$. If $\mathbf{ Q}_0=0$, then the corresponding quaternion matrix is called a pure quaternion matrix.

\section{Dynamic mode decomposition in real field}

DMD is a data-driven method which spatiotemporally decomposes data from snapshots or measurements of a given system in time into a set of dynamic modes. In this section, we will give a brief review of the DMD theory.

 It is assumed here that the number of data points collected at a given time is $n$, with the number of unified samples being $m$, and the timestep is denoted by $\Delta t$. We use a vector $\mathbf{x}_l \in \mathbb{R}^n$ to denote the $n$ data points collected at the time $t_l$, $l=1,2,\dots, m$. Then the data can be arranged into two matrices, $\mathbf{X}$, $\mathbf{Y}$:
\[\mathbf{X}=
 \begin{bmatrix}
  \mathbf{x}_1 & \mathbf{x}_2 & \mathbf{x}_3 & \cdots & \mathbf{x}_{m-1}
 \end{bmatrix},\]
\[ \mathbf{Y}=
 \begin{bmatrix}
  \mathbf{x}_2 & \mathbf{x}_3 & \mathbf{x}_4 & \cdots & \mathbf{x}_{m}
 \end{bmatrix},\]
 where $\mathbf{X}, \mathbf{Y} \in \mathbb{R}^{n \times {(m-1)}}$. Assuming that there exists a linear operator $\mathbf{A} \in \mathbb{R}^{n \times n}$ which describes the dynamic change between the data at time $l$ and the data at time $l+1$ such that $\mathbf{x}_{l+1}=\mathbf{A}\mathbf{x}_{l} $. Based on the assumption, the best fit linear map $\mathbf{A}$ which maps $\mathbf{X}$ to $\mathbf{Y}$ is defined as:
 \begin{equation}
\mathbf{A} = \mathop{\text{argmin}}\limits_{\mathbf{A}} \left\| \mathbf{Y}-\mathbf{A}\mathbf{X} \right\|_{F}= \mathbf{Y} \mathbf{X}^{\dagger},  \label{realpse}
\end{equation}
 where $\mathbf{X}^\dagger$ is the pseudoinverse of $\mathbf{X}$, $\left\| \cdot \right\|_{F}$ denotes the Frobenius norm. The eigenvectors and eigenvalues of $\mathbf{A}$ are defined as the DMD modes and eigenvalues, respectively. Substituting $\mathbf{X}^\dagger = \mathbf{X}^T(\mathbf{X}\mathbf{X}^T)^{-1}$ into Eq. (\ref{realpse}) yields 
 \begin{equation}
\mathbf{A} = (\mathbf{Y}\mathbf{X}^T) (\mathbf{X}\mathbf{X}^T)^{-1}, 
\end{equation}
 which means the need of the calculation of the pseudoinverse of an $n \times n$ matrix rather than the pseudoinverse of an $n \times m$ matrix. However, in many systems, the dimension $n$ of the system is larger than the number of snapshots $m$.
 
 Therefore, for a high-dimensional data $\mathbf{x} \in \mathbf{R}^n$ (i.e. $n$ is large), instead of explicitly computing $\mathbf{A}$ to obtain its dominant eigenvalues and eigenvectors, the DMD method obtains the dominant eigenvalues and eigenvectors of $\mathbf{A}$ by implementing dimensionality reduction. There are many variants of the DMD. The main difference between those methods is the way to calculate the DMD modes and eigenvalues. In this paper, we will expand the exact DMD to the quaternion system, so we give a brief review of the exact DMD in this section.

 \begin{algorithm}[htb]  
  \caption{ Exact DMD algorithm \cite{tu2013dynamic}.}  
  \label{alg:Framwork}  
  \begin{algorithmic}[1]  
    \Require  
      Two matrices $\mathbf{X}$ and $\mathbf{Y}$ which are constructed from the data. 
    \Ensure  
      DMD modes ($\mathbf{\Phi}$) and DMD eigenvalues.   
    \State Calculate the (reduced) SVD of the matrix $\mathbf{X}$, i.e., find $\mathbf{U}$, $\mathbf{\Sigma}$ and $\mathbf{V}$ such that
    \begin{equation}
    \mathbf{X} = \mathbf{U} \mathbf{\Sigma}\mathbf{V}^{T},
    \end{equation}
    where $\mathbf{U} \in \mathbb{R}^{n \times r}$, $\mathbf{\Sigma} \in \mathbb{R}^{r \times r}$, and $\mathbf{V} \in \mathbb{R}^{m \times r}$, and $r$ is the rank of $\mathbf{X}$.
    \label{code:fram:extract}  
    \State Define $\widetilde{\mathbf{A}}= \mathbf{U}^{T}\mathbf{Y}\mathbf{V}\mathbf{\Sigma}^{-1}$.  
    \label{code:fram:trainbase}  
    \State Compute the spectral decomposition of $\widetilde{\mathbf{A}}$, obtaining $\widetilde{\mathbf{W}}$ and  $\widetilde{\Lambda}$ such that
    \begin{equation}
    \widetilde{\mathbf{A}} \widetilde{\mathbf{W}} = \widetilde{\mathbf{W}} \widetilde{\Lambda},
    \end{equation}
    where $\widetilde{\Lambda}$ is a diagonal matrix and its diagonal elements are the eigenvalues $\lambda_{v}$ of both the matrix $\widetilde{\mathbf{A}}$ and the matrix $\mathbf{A}$, i.e., the DMD eigenvalues, while the columns of the matrix $\widetilde{\mathbf{W}}$ are the corresponding eigenvectors of the matrix $\widetilde{\mathbf{A}}$.
    \label{code:fram:add}  
    \State Calculate the DMD modes $\mathbf{\Phi}$ by using the eigenvectors $\widetilde{\mathbf{W}}$ of the matrix $\widetilde{\mathbf{A}}$, given by
     \begin{equation}
   \mathbf{\Phi}= \mathbf{Y}\mathbf{V}\mathbf{\Sigma}^{-1}\widetilde{\mathbf{W}}.
    \end{equation}
    \label{code:fram:classify}  \\
    \Return $\mathbf{\Phi}$ and $\widetilde{\Lambda}$;  
  \end{algorithmic}  
\end{algorithm}  
 
 After obtaining the low-rank approximations of the eigenvalues and eigenvectors, i.e., the DMD eigenvalues and the DMD modes, the data $\mathbf{x}_{\text{DMD}}$ at time $t$ for any time after the data vector $\mathbf{x}_1$ was collected ($t_1=0$) can be obtained by
$\mathbf{x}_{\text{DMD}} = \mathbf{A}^{t}\mathbf{x}_1.$ For the convenience of spectral expansion in continuous time, a mapping 
of $\lambda_v$ is defined as $\omega_v = \text{log}(\lambda_v)/\Delta t$.
Based on the obtained approximate spectral decomposition of the operator $\mathbf{A}$, the approximate system state at all future times, $\mathbf{x}_{\text{DMD}}$, is given by
 \begin{equation}
   \mathbf{x}_{\text{DMD}}(t) = \sum_{s=1}^{r}b_s\varphi_s e^{\omega_{s}t} = \mathbf{\Phi}\mathbf{\Omega}^t\mathbf{b},
    \end{equation}
  where 
  \[ \mathbf{\Omega}=
 \begin{bmatrix}
   e^{\omega_{1}} & 0 & \cdots & 0\\
   0 & e^{\omega_{2}} & \ddots & 0\\
   \vdots & \ddots & \ddots & \vdots \\
   0 & 0 & \cdots & e^{\omega_{r}}
 \end{bmatrix},\]  
 and $\mathbf{b}$ is generally calculated as 
 \begin{equation}
 \mathbf{b} = \mathbf{\Phi}^\dagger \mathbf{x}_1= (\mathbf{\Phi}^{T} \mathbf{\Phi})^{-1}\mathbf{\Phi}^{T}\mathbf{x}_1.
 \end{equation}
 The method is shown in Algorithm \ref{alg:Framwork} .

\section{Proposed quaternion-based dynamic mode decomposition}
In this section, we will first introduce the calculation of eigenvalues, eigenvectors and singular  value decomposition of a quaternion matrix. Then we will extend DMD to the quaternion number system. The purpose is to use the advantages of quaternion to represent color images.

\subsection{Eigenvalues and eigenvectors of a quaternion matrix}

Due to non-commutation of quaternion, there are two categories of eigenvalues of a quaternion matrix, left eigenvalues and right eigenvalues. In this paper, we only use the right eigenvalues, so a brief review of right eigenvalues and eigenvectors of a quaternion matrix will be given. For convenience, we abbreviate right eigenvalues as the eigenvalues. For detailed overview of eigenvalues and eigenvectors of a quaternion matrix, we recommend \cite{lee1948eigenvalues}\cite{zhang1997quaternions}.\\

\noindent
\begin{definition} \label{Righteig} (Eigenvalues \cite{zhang1997quaternions})  A quaternion $\dot \lambda$ is said to be a right (left) eigenvalue of a quaternion matrix $\mathbf{\dot{Q}}$ if it satisfies 
\begin{equation}
    \mathbf{\dot{Q}}\mathbf{\dot{x}}=\mathbf{\dot{x}}\dot{\lambda} \quad ( \mathbf{\dot{Q}}\mathbf{\dot{x}}=\dot{\lambda} \mathbf{\dot{x}}).
\end{equation}
\end{definition}

\noindent
\begin{definition} \label{Righteigclass} (Eigenvalue class \cite{chang2003quaternion}) If $\dot \lambda$ is one right eigenvalue of a quaternion matrix $\mathbf{\dot{Q}} \in \mathbb{H}^{N \times N}$, then every element of the set $\Gamma \equiv \{\dot q \dot \lambda \dot q^{-1}: \text{where} \: \dot q \: \text{is any unit quaternion with its norm} \: \lvert \dot q \rvert = 1\}$ is also a right eigenvalue of $\mathbf{\dot{Q}}$. Moreover, a single eigenvalue $\lambda _c \in C^{+}$ will be contained in this set $\Gamma$ , so this set is considered as the eigenvalue class of $\lambda _c$. \end{definition}
 
 Therefore, the eigenvalues of a quaternion matrix are infinite, and the eigenvalues of a quaternion matrix are finite if and only if all eigenvalues of this quaternion matrix are real, however, there are finite eigenvalue classes.

 \noindent
 \begin{theorem} \label{standard eigenvalues} (The standard eigenvalues of a quaternion matrix \cite{zhang1997quaternions}) Any $N\times N$ quaternion matrix $\mathbf{\dot{Q}}$ has exactly $N$ right eigenvalues which are complex numbers with nonnegative imaginary parts. Those eigenvalues are defined as the standard eigenvalues of the quaternion matrix $\mathbf{\dot{Q}}$.
 \end{theorem}

 \noindent
\begin{definition} \label{Complex representation} (The Complex representation of a quaternion matrix) Given a quaternion matrix $\mathbf{\dot{Q}} \in \mathbb{H}^{M \times N}$, and let $\mathbf{\dot{Q}}=\mathbf{Q_a}+\mathbf{Q_b}j$, where $\mathbf{Q_a}$, $\mathbf{Q_b} \in \mathbb{C}^{M \times N}$, then the complex representation of $\mathbf{\dot{Q}}$ is defined as \cite{zhang1997quaternions} 
\begin{equation}
   \mathbf{\chi_{\dot{Q}}} =
\begin{bmatrix}
 \mathbf{Q_a}  & \mathbf{Q_b} \\
    -{\mathbf{Q_b}}^\ast & {\mathbf{Q_a}}^\ast
\end{bmatrix},
\end{equation}
where $\mathbf{\chi_{\dot{Q}}} \in \mathbb{C}^{2M \times 2N}$.
\end{definition}

There are many similar properties between the quaternion matrix and its corresponding complex representation matrix and the detail can be found in \cite{zhang1997quaternions}.

Based on the eigenvalues and eigenvectors of $\mathbf{\chi_{\dot{Q}}}$, we can compute the eigenvalues and eigenvectors of $\mathbf{\dot{Q}}$, which is presented in Theorem \ref{Th2}.

 \noindent
 \begin{theorem} \label{Th2} (The calculation of Standard eigenvalues of the quaternion matrix \cite{chang2003quaternion})  Given a quaternion matrix $\mathbf{\dot{Q}} \in \mathbb{H}^{N \times N}$, and let $\mathbf{\dot{Q}}=\mathbf{Q_a}+\mathbf{Q_b}j$, then the complex eigenvalues of $\mathbf{\dot{Q}}$ are the same as the eigenvalue of $\mathbf{\chi_{\dot{Q}}}$. Further, the complex eigenvalues of $\mathbf{\chi_{\dot{Q}}}$ appear in conjugate pairs. Specially, if $\mathbf{\chi_{\dot{Q}}}$ has any real eigenvalue, it occurs an even number of times. Therefore, $N$ complex eigenvalues with nonnegative imaginary part of the quaternion matrix $\mathbf{\dot{Q}}$ can be obtained.
 \end{theorem}
 
 The relation between the eigenvectors of the quaternion matrix $\mathbf{\dot{Q}}$ and the eigenvectors of $\mathbf{\chi_{\dot{Q}}}$ is that if $(\mathbf{x})_{2N \times 1}=\begin{bmatrix}
 (\mathbf{x}_1)_{N \times 1}  \\ (\mathbf{x}_2)_{N \times 1}
\end{bmatrix}$ is an eigenvector of the complex matrix $\mathbf{\chi_{\dot{Q}}}$ corresponding to eigenvalue $\lambda$ of $\mathbf{\chi_{\dot{Q}}}$, then $(\mathbf{\dot{x}})_{N \times 1}=\mathbf{x}_1-\mathbf{x}_2^{\ast} \cdot j$ is an eigenvector of the quaternion matrix $\mathbf{\dot{Q}}$ corresponding to eigenvalue $\lambda$ of $\mathbf{\dot{Q}}$, where $\mathbf{x}_1$, $\mathbf{x}_2 \in \mathbb{C}^{N \times 1}$, and $\mathbf{\dot{x}} \in \mathbb{H}^{N \times 1}$.

\subsection{Singular value decomposition of a quaternion matrix}
 \noindent
\begin{definition} \label{rank of quaternion} (The rank of quaternion matrix \cite{zhang1997quaternions}) The rank of a quaternion matrix $\mathbf{\dot{Q}} \in \mathbb{H}^{M \times N}$ is the maximum number of right (left) linearly independent columns (rows) of $\mathbf{\dot{Q}}$.
\end{definition}

\noindent
\begin{theorem} \label{QSVD} (Singular value decomposition of a quaternion matrix (QSVD) \cite{zhang1997quaternions}) Given any quaternion matrix $\mathbf{\dot{Q}} \in \mathbb{H}^{M \times N}$ of rank $r$, there exist two quaternion unitary matrices $\mathbf{\dot{U}} \in \mathbb{H}^{M \times M}$ and $\mathbf{\dot{V}} \in \mathbb{H}^{N \times N}$ such that
\begin{equation}
    \mathbf{\dot{Q}} = \mathbf{\dot{U}}\begin{bmatrix}
 \mathbf{\Sigma}_r  & \mathbf{0} \\
    \mathbf{0} & \mathbf{0}
\end{bmatrix}\mathbf{\dot{V}}^H=\mathbf{\dot{U}}\mathbf{\Lambda} \mathbf{\dot{V}}^H,\end{equation}
where $\mathbf{\Sigma}_r$ is a real diagonal matrix with $r$ positive entries on its diagonal (i.e. singular values of $\mathbf{\dot{Q}}$).
\end{theorem}

The computation of $\mathbf{\dot{U}}$, $\mathbf{\dot{V}}$ and the singular values of $\mathbf{\dot{Q}}$ can be obtained based on the SVD of its complex representation $\mathbf{\chi_{\dot{Q}}}$. The calculation of QSVD is briefly summarized as follows \cite{xu2015vector}: 

\noindent
\begin{enumerate}
\item Compute the SVD of $\mathbf{\chi_{\dot{Q}}}$, and here we denote $\mathbf{\chi_{\dot{Q}}}=\mathbf{U} {\mathbf{\Lambda}}^{\mathbf{\chi_{\dot{Q}}}} \mathbf{V}^{H}$.
\item Then, we can get that
\begin{equation}
    \begin{cases}
    \mathbf{\Lambda} = \text{row}_{odd}(\text{col}_{odd}({\mathbf{\Lambda}}^{\mathbf{\chi_{\dot{Q}}}})),\\
    \mathbf{\dot{U}}=\text{col}_{odd}(\mathbf{U_1})+\text{col}_{odd}(-(\mathbf{U_2})^{\ast})j ,\\
    \mathbf{\dot{V}}=\text{col}_{odd}(\mathbf{V_1})+\text{col}_{odd}(-(\mathbf{V_2})^{\ast})j,
    \end{cases}
\end{equation}
\end{enumerate}
\noindent
where 
\begin{equation*}
\mathbf{U}=
   \begin{bmatrix}
    (\mathbf{U}_1)_{M \times 2M}\\
     (\mathbf{U}_2)_{M \times 2M}
    \end{bmatrix}, \quad
    \mathbf{V}=\begin{bmatrix}
    (\mathbf{V}_1)_{N \times 2N}\\
     (\mathbf{V}_2)_{N \times 2N}
    \end{bmatrix},
\end{equation*}
and $\text{row}_{odd}(\mathbf{M})$ and $\text{col}_{odd}(\mathbf{M})$ represent the extraction of the odd rows and odd columns of matrix $\mathbf{M}$ respectively.

\subsection{Quaternion-based dynamic mode decomposition}
Assume that there are $m$ snapshots of the state of a dynamic system and each snapshot is arranged into an $n \times 1$ quaternion vector with the form
\[\mathbf{\dot x}(t_l), \quad \mathbf{\dot y}(t_l)  \in \mathbb{H}^{n \times 1},\]
where $\mathbf{\dot y}(t_l) = \mathbf{F}(\mathbf{\dot x}(t_l))$, $l=1, 2, \dots, m$. These snapshots can form two data matrices, $\mathbf{\dot X}$ and $\mathbf{\dot Y} \in \mathbb{H}^{n \times m}$:
\[\mathbf{\dot X}=
 \begin{bmatrix}
  \mathbf{\dot x}(t_1) & \mathbf{\dot x}(t_2) & \mathbf{\dot x}(t_3) & \cdots & \mathbf{\dot x}_{(t_{m-1})}
 \end{bmatrix},\]
\[ \mathbf{\dot Y}=
 \begin{bmatrix}
  \mathbf{\dot y}(t_1) & \mathbf{\dot y}(t_2) & \mathbf{\dot y}(t_3) & \cdots & \mathbf{\dot y}(t_{m-1})
 \end{bmatrix}.\]
If the data was collected by uniform sampling in time, then we have $t_l = l \Delta t$, where $\Delta t$ is the timestep. And it is assumed here that $t_1=0$. Similarly, the purpose of quaternion-based DMD is to find the leading spectral decomposition of the best-fit linear operator $\mathbf{\dot Q}$ that reflects the changes of the two matrices $\mathbf{\dot X}$ and $\mathbf{\dot Y}$:
\begin{equation}
    \mathbf{\dot Y} \thickapprox \mathbf{\dot {Q}}\mathbf{\dot X}.
\end{equation}
The definition of the best-fit operator $\mathbf{\dot Q}$ in mathematics is 
\begin{equation}
    \mathbf{\dot {Q}}  = \mathop{\text{argmin}}\limits_{\mathbf{\dot {Q}} } \left\| \mathbf{\dot Y}-\mathbf{\dot {Q}}\mathbf{\dot X} \right\|_{F}= \mathbf{\dot Y} \mathbf{\dot X}^{\dagger}=\mathbf{\dot Y} \mathbf{\dot X}^{H}(\mathbf{\dot X}\mathbf{\dot X}^{H})^{-1}, 
\end{equation}
where $\mathbf{\dot X}^\dagger$ is the quaternionic pseudoinverse of $\mathbf{\dot X}$, and $(\mathbf{\dot X}\mathbf{\dot X}^{H})^{-1}$ is computed by using QSVD of $\mathbf{\dot X}\mathbf{\dot X}^{H}$, that is, after obtaining the QSVD of $\mathbf{\dot X}\mathbf{\dot X}^{H}$, then replace all nonzeros singular values by their reciprocals. $\left\| \cdot \right\|_{F}$ denotes the Frobenius norm.  The Frobenius norm of a quaternion matrix $\mathbf{\dot {Q}} \in \mathbb{H}^{N_1 \times M_1} $ is defined as \cite{zhang1997quaternions}:
$\left\| \mathbf{\dot {Q}} \right\|_{F}=\sqrt {\sum_{n_1=1}^{
N_1}\sum_{m_1=1}^{
M_1}\lvert {\dot {q}}_{n_1m_1}\rvert ^{2}} = \sqrt{\text{tr}((\mathbf{\dot {Q}})^{H}\mathbf{\dot {Q})}}.$
When the dimension of per time snapshot $n$ is large, it is difficult to deal with the quaternion matrix $\mathbf{\dot{Q}}$ directly. Therefore, we also reduce the dimension of the quaternion matrix $\mathbf{\dot{Q}}$ to obtain its leading spectral decomposition. Before that, we first establish the spectral decomposition of a quaternion matrix.
\noindent
\begin{theorem} \label{quaspe} (Spectral decomposition of a quaternion matrix)  Given any quaternion matrix $\mathbf{\dot{Q}} \in \mathbb{H}^{M \times M}$, after calculating the standard eigenvalues and eigenvectors of $\mathbf{\dot{Q}}$, there exist two quaternion matrices $\mathbf{\dot{\Phi}} \in \mathbb{H}^{M \times M}$ and $\mathbf{\dot{\Lambda}} \in \mathbb{H}^{M \times M}$ such that
\begin{equation}
    \mathbf{\dot{Q}} = \mathbf{\dot{\Phi}}\mathbf{\dot{\Lambda}}\mathbf{\dot{\Phi}}^{\dagger} ,
    \end{equation}
where $\mathbf{\dot{\Lambda}}$ is a quaternion diagonal matrix with $M$ quaternion numbers on its diagonal, i.e., $M$ standard eigenvalues of $\mathbf{\dot{Q}}$, and the columns of $\mathbf{\dot{\Phi}}$ are the corresponding eigenvectors. \end{theorem}
\begin{proof}
By using Theorem \ref{Th2}, we can calculate the $M$ standard eigenvalues of $\mathbf{\dot{Q}}$ and its corresponding $M$ eigenvectors. We denote the $v$-th eigenvalue of $\mathbf{\dot{Q}}$ as $\dot \lambda _v$, and $\boldsymbol{\dot \phi }_v$ is the eigenvector corresponding to $\dot \lambda _v$. Then the $M$ eigenvectors can form a quaternion matrix which is denoted as $\mathbf{\dot \Phi}$, i.e., $\mathbf{\dot \Phi}=\begin{bmatrix}
 \boldsymbol{\dot \phi} _1 & \boldsymbol{\dot \phi }_2 & \cdots & \boldsymbol{\dot \phi }_M
\end{bmatrix}.$
Meanwhile another diagonal quaternion matrix $\mathbf{\dot{\Lambda}}$ can be formed, that is $\mathbf{\dot{\Lambda}} = \text{diag}(\dot \lambda _1, \dot \lambda _2, \cdots, \dot \lambda _M),$ where $\mathbf{\dot{\Lambda}}$ is a quaternion diagonal matrix and has the $M$ standard eigenvalues of $\mathbf{\dot{Q}}$ on its diagonal.
Therefore, based on Definition \ref{Righteig}, we have
\begin{equation}
    \mathbf{\dot{Q}} \mathbf{\dot \Phi}=\mathbf{\dot \Phi}\mathbf{\dot{\Lambda}} \label{eqspe}.
\end{equation}
Furthermore, The following formula can be obtained by multiplying both sides of the Eq. (\ref{eqspe}) by the pseudoinverse of the quaternion matrix $\mathbf{\dot \Phi}$, that is
\begin{equation}
    \mathbf{\dot{Q}} = \mathbf{\dot{\Phi}}\mathbf{\dot{\Lambda}}\mathbf{\dot{\Phi}}^{\dagger}.
\end{equation}
\end{proof}

\begin{algorithm}[htb]  
  \caption{ Quaternion-based DMD algorithm (Q-DMD).}  
  \label{qdmd_method}  
  \begin{algorithmic}[1]  
    \Require  
      Two matrices $\mathbf{\dot X}$ and $\mathbf{\dot Y}$ which are constructed from data. 
    \Ensure  
      Q-DMD modes ($\mathbf{\dot {\Phi}}$) and Q-DMD eigenvalues.   
    \State Calculate the (reduced) QSVD of the matrix $\mathbf{\dot X}$, i.e. find $\mathbf{\dot U}$, $\mathbf{\Sigma}$ and $\mathbf{\dot V}$ such that
    \begin{equation}
    \mathbf{\dot X} = \mathbf{ \dot U} \mathbf{\Sigma}\mathbf{\dot V}^{H},
    \end{equation}
    where $\mathbf{\dot U} \in \mathbb{H}^{n \times r}$, $\mathbf{\Sigma} \in \mathbb{R}^{r \times r}$, and $\mathbf{\dot V} \in \mathbb{H}^{m \times r}$, and $r$ is the rank of $\mathbf{\dot X}$.
    \label{code:fram:extract}  
    \State Define $\dot{\widetilde{\mathbf{Q}}}= \mathbf{\dot U}^{H}\mathbf{\dot Y}\mathbf{\dot V}\mathbf{\Sigma}^{-1}$.  
    \label{code:fram:trainbase}  
    \State Compute the right spectral decomposition of $\dot{\widetilde{\mathbf{Q}}}$ using Theorem \ref{quaspe}, obtaining $\dot{\widetilde{\mathbf{W}}}$ and  $\dot{\widetilde{\Lambda}}$ such that
    \begin{equation}
    \dot{\widetilde{\mathbf{Q}}} \dot{\widetilde{\mathbf{W}}} = \dot{\widetilde{\mathbf{W}}} \dot{\widetilde{\Lambda}},
    \end{equation}
    where $\dot{\widetilde{\Lambda}}$ is a diagonal matrix and its diagonal elements are the eigenvalues ${\dot \lambda_{k}} \in \mathbb{H}$ of both the matrix $\dot{\widetilde{\mathbf{Q}}}$ and the matrix $\dot{\mathbf{Q}}$, i.e. the Q-DMD eigenvalues, while the columns of the matrix $\dot{\widetilde{\mathbf{W}}}$ are the corresponding eigenvectors of the matrix $\dot{\widetilde{\mathbf{Q}}}$.
    \label{code:fram:add}  
    \State Calculate the Q-DMD modes $\dot{\mathbf{\Phi}}$ by using the eigenvectors $\dot{\widetilde{\mathbf{W}}}$ of the matrix $\dot{\widetilde{\mathbf{Q}}}$, given by
     \begin{equation}
   \dot{\mathbf{\Phi}}= \dot{\mathbf{Y}}\dot{\mathbf{V}}\mathbf{\Sigma}^{-1}\dot{\widetilde{\mathbf{W}}}.
    \end{equation}
    \label{code:fram:classify}  \\
    \Return $\dot{\mathbf{\Phi}}$ and $\dot{\widetilde{\Lambda}}$;  
  \end{algorithmic}  
\end{algorithm}
Now, we calculate the quaternion-based DMD (Q-DMD) modes and eigenvalues, and the calculation process is shown in Algorithm \ref{alg:Framwork2}. After calculating the low-rank approximations of eigenvalues (i.e., Q-DMD eigenvalues) and eigenvectors (i.e., Q-DMD modes) of the quaternion matrix $\dot{\mathbf{Q}}$, the system state can be expanded by using the spectral decomposition:
\begin{equation}
    \mathbf{\dot {x}}_{\text{Q-DMD}}(t)=\sum_{s=1}^{r}\mathbf{\boldsymbol{\dot{\phi}}}_s \dot{\lambda }_s^{t}\dot{b}_s= \mathbf{\dot{\Phi}} \mathbf{\dot{\Lambda}}^{t}\mathbf{\dot{b}},
\end{equation}
where $\mathbf{\dot{b}} =\mathbf{\dot{\Phi}}^{\dagger}\mathbf{\dot {x}}_1 $ contains the initial amplitudes for the modes, and $\mathbf{\dot{\Lambda}}^{t} = \text{diag}(\dot \lambda _1, \dot \lambda _2, \cdots, \dot \lambda _r)$. The Q-DMD eigenvalues can be converted to the continuous form by using Definition \ref{logarithm} :
\begin{equation}
    \dot \omega _v = \frac{\text{ln}(\dot \lambda_v)}{\Delta t}.
\end{equation}
Therefore, the right spectral decomposition above can be rewritten in continuous time
\begin{equation}
    \mathbf{\dot {x}}_{\text{Q-DMD}}(t)=\sum_{s=1}^{r}\mathbf{\boldsymbol{\dot{\phi}}}_s e^{\dot \omega_s  t}_s\dot{b}_s= \mathbf{\dot{\Phi}} \mathbf{\dot{\Omega}}^t\mathbf{\dot{b}}, \label{quadmd}
\end{equation}
where \[ \mathbf{\dot \Omega}=
 \begin{bmatrix}
   e^{\dot{\omega}_{1}} & 0 & \cdots & 0\\
   0 & e^{\dot{\omega}_{2}} & \ddots & 0\\
   \vdots & \ddots & \ddots & \vdots \\
   0 & 0 & \cdots & e^{\dot{\omega}_{r}}
 \end{bmatrix}.\]

 \subsection{Background modeling using the quaternion-based DMD method}
 The Q-DMD method can deal with a color video sequence because the frames of the video are normally uniformly sampled, and a color pixel can be represented by a quaternion. It is assumed here that the color video contains $m$ frames, and each frame has $n$ pixels in total, then each frame can be vectorized into a pure quaternion vector denoted as
\[\mathbf{\dot{x}}_s= R(:,s)\cdot i+G(:,s)\cdot j +B(:,s)\cdot k, \] 
where $\mathbf{\dot{x}}_s \in \mathbb{H}^{n \times 1}$, and $s=1,2, \dots, m$. $R(:,s)$, $G(:,s)$, and $B(:,s) \in \mathbb{R}^{n \times 1}$ respectively represent the pixel values of the corresponding red, green and blue channels after vectorization of the $s$-th frame. Then, all frames of the color video are vectorized and form a quaternion matrix $\mathbf{\dot Q}=\begin{bmatrix} 
 \mathbf{\dot{x}} _1 & \mathbf{\dot{x}}_2 & \cdots & \mathbf{\dot{x}}_m
\end{bmatrix}  \in \mathbb{H}^{n \times m}$. And any given frame of the color video can be reconstructed by using Q-DMD method, that is by using Eq. (\ref{quadmd}). Given a color video sequence and assuming that frames are collected at uniform intervals, then the time points of these collected frames can form a vector $\mathbf{t}=
 \begin{bmatrix}
   t_1 & t_2 & \cdots & t_m
 \end{bmatrix}$. For convenience, the time point of collecting the first frame is recorded as $0$, and so on. The time point of collecting the $m$-th frame is recorded as $m-1$. Therefore, using the Q-DMD method, the full video sequence $\mathbf{\dot{X}}$ can be reconstructed by the following formula
 \begin{equation}
     \mathbf{X}_{\text{Q-DMD}} = \sum_{s=1}^{r}{\boldsymbol{\dot{\phi}}_s}e^{\dot{\omega}_s \mathbf{t}}\dot{b}_s=\mathbf{\dot{\Phi}} \mathbf{\dot{\Omega}^t}\mathbf{\dot{b}},\label{wholedmd}
 \end{equation}
where $\mathbf{\boldsymbol{\dot{\phi}}}_s \in \mathbb{H}^{n \times 1}$, $\mathbf{t} \in \mathbb{R}^{1 \times m}$. 
After obtaining the prediction of the whole video sequences by using Eq. (\ref{wholedmd}), the separation of foregrounds and the background can be obtained by thresholding the low-frequency modes of the corresponding eigenvalues. Generally, any part of the first video frame that does not change with time, or changes very slowly in time satisfies $\lvert{\dot{\omega}_p}\rvert \approx 0$, which means the Fourier mode corresponding to this part is near the origin of the quaternion space. Therefore, assume there exists $\dot{\omega}_p$ satisfies $\lvert{\dot{\omega}_p}\rvert \approx 0$, where $p \in \{1,2, \dots, r\}$, and 
the other Fourier modes (i.e. $\dot{\omega}_s$, $\forall \, s \neq p$) are not near the origin. Then we can obtain that
\begin{equation}
\begin{aligned}
    \mathbf{X}_{\text{Q-DMD}} &=\textbf{L} + \textbf{S}\\
    &=\underbrace{{\boldsymbol{\dot{\phi}}_p}e^{\dot{\omega}_p \mathbf{t}}\dot{b}_p}\limits_{\text{Background}}+\underbrace{\sum_{s \neq p}\mathbf{\boldsymbol{\dot{\phi}}}_s e^{\dot{\omega}_s \mathbf{t}}\dot{b}_s} \limits_{\text{Foreground}}. \label{sepmet}
    \end{aligned}
\end{equation}
The parameter $r$ is related to the dimensionality reduction and it is fixed to $m-1$, i.e., one less than the number of frames of the color video sequence in this paper. It is worth noting that quaternions do not satisfy the commutative law, so the product of the above Eq. (\ref{sepmet}) cannot change the order, and this is different from the real field generally, that is
 \[\mathbf{X}_{\text{Q-DMD}} = \boldsymbol{\dot{\phi}}_p e^{\dot{\omega}_p \mathbf{t}}\dot{b}_p+\sum_{s \neq p}\boldsymbol{\dot{\phi}}_s e^{\dot{\omega}_s \mathbf{t}}\dot{b}_s  \neq \dot{b}_p \boldsymbol{\dot{\phi}}_p e^{\dot{\omega}_p} +\sum_{s \neq p}\dot{b}_s\mathbf{\boldsymbol{\dot{\phi}}_s}e^{\dot{\omega}_s \mathbf{t}}.\]
 There exists another difference between DMD and Q-DMD. If we denote that the full video sequence reconstructed by DMD as $ \mathbf{X}$ and assume $\mathbf{X} \in \mathbb{R}^{n \times m}$, then each term of the DMD reconstruction is a complex matrix, i.e., $b_s \phi_s e^{{\omega}_s \mathbf{t}} \in \mathbf{C}^{n \times m}$, while when the DMD term is decomposed into approximate low rank and sparse components, real value output is required. However, each term of the Q-DMD reconstruction is a quaternion matrix, i.e., $\boldsymbol{\dot{\phi}}_s e^{\dot{\omega}_s \mathbf{t}}\dot{b}_s \in \mathbb{H}^{n \times m}$, and the coefficient matrices corresponding to the three imaginary units $i, j, k$ are real matrix and correspond to the three channels RGB of the color frame to be reconstructed.

 \section{Experimental results and discussion}
 \indent
 
 In this section, We conduct numerical experiments on two publicly available benchmark datasets, i.e., Scene background modeling dataset (SBMnet dataset) $\footnote{http://www.SceneBackgroundModeling.net}$ \cite{jodoin2017extensive} and Scene background initialization dataset (SBI dataset) $\footnote{http://sbmi2015.na.icar.cnr.it/SBIdataset.html}$ \cite{maddalena2015towards, bouwmans2017scene}, and illustrate that quaternion-based has the ability to separate color video sequences into background and foreground. We compare our method with the original DMD and several state-of-the-art approaches.
 
 \textbf{Quantitative assessment:} To evaluate the performance of our proposed method, we employ the following six metrics to measure the performance:
\begin{enumerate}[1)]
\item Average gray-level error (AGE): The average of the gray-level absolute difference between ground truth (GT) and the computed background (CB) image.
\item Percentage of error pixels (pEPs): The percentage of error pixels (number of pixels in CB whose value differs from the value of the corresponding pixel in GT by more than a threshold) with respect to the total number of pixels in the image.
\item Percentage of clustered error pixels (pCEPS): The percentage of clustered error pixels (number of pixels whose 4-connected neighbors are also error pixels) with respect to the total number of pixels in the image.
\item MultiScale Structural Similarity Index (MSSSIM): Estimate of the perceived visual distortion.
\item Peak-signal-to-noise-ratio (PSNR): Amounts to $10\text{log}_{10}((L-1)^2/\text{MSE})$ where $L$ is the maximum number of grey levels and MSE is the mean squared error between GT and CB images.
\item Color image quality measure (CQM): Based on a reversible transformation of the YUV color space and on the PSNR computed in the single YUV bands. It assumes values in db and the higher the CQM value, the better is the background estimate.
\end{enumerate}

\textbf{Experiment on SBMnet dataset:} We first simulated on SBMnet dataset to evaluate the proposed Q-DMD method for color video background modeling. The dataset consists of a wide range of challenging videos, such as camera jitter, intermittent object motion, background motion, abandoned object, illumination changes, long and short sequences of images and the ground truth of some videos was provided. The first frame of the eight videos from the SBMnet dataset are shown in Figure \ref{fig:SBM_input}. The spatial resolutions of those videos vary from $240 \times 240$ to $800\times 600$. The spatial resolutions of the videos "Basic (511) " and "IntermittenMotion (2007)" are $480\times 640$ and $576\times 720$, respectively, so the two videos were down-sampled by a factor of 2 to make the computational memory requirements manageable for personal computers. In order to process a color video efficiently, for the color video whose total number of frames in the SBMnet database is far more than 200 frames, we randomly selected 200 consecutive frames containing the foreground and background information for experiments and the frames that did not change over time were trimmed. We compared the performance of our method with the DMD, and several other existing state-of-the-art approaches, including MSCL \cite{javed2017background}, FSBE \cite{djerida2019robust}, LaBGen \cite{laugraud2017labgen}, NExBI \cite{mseddi2019real}, Photomontage \cite{agarwala2004interactive}, Bidirectional Analysis \cite{minematsu2016background}, BE-AAPSA \cite{ramirez2017temporal}, FC-FlowNet \cite{halfaoui2016cnn}, RMR \cite{ortego2016rejection}. 

 \indent
The results of the generated backgrounds based on Q-DMD on SBMnet dataset are displayed in Figure \ref{fig:SBM}, and the quantitatively evaluated results are shown in Table \ref{dmd and qdmd sbm} and \ref{qdmd and other sbm}. We use six evaluation metrics (AGE, pEPs, pCEPs, MSSSIM, PSNR and CQM) to demonstrate the performance of the proposed Q-DMD method. Numbers that are bold and underlined, numbers that are bold and dashed, and numbers that are bold represent the highest, second and third highest CQM values on each video, respectively. It can be found from Table \ref{dmd and qdmd sbm} that the performance of our Q-DMD method is better than that of DMD method in the vast majority of images. The CQM is a color image quality measure based on reversible luminance and chrominance (YUV) color transformation and PSNR measure \cite{yalman2013new}. Therefore, in addition to the advantage on PSNR, the good performance of our method on CQM also demonstrates the advantages of the quaternion-based model.

\indent
As can be seen from Table \ref{qdmd and other sbm}, our method achieves competitive CQM values in four videos ("511", "boulevardJam", "boulevard" and "Badminton") compared with other methods. For those videos with very short background exposure duration, such as "Board", "AVSS2007", "boulevardJam" and "Badminton", as shown in Figure \ref{fig:SBM}, there exists the shadow effect of moving objects in the generated background, which leads to a reduction in performance. However, for long sequence images which the background appears for a long time, such as "Basic (511)", "Jitter (boulevard)", the CQM value obtained by our method is almost the same as that obtained by FC-FlowNet method. In particular, our method achieved the second highest CQM value on "Clutter (boulevardJam)" and "Jitter (boulevard)". 
These results show that compared with the state-of-the-art approaches, Q-DMD method is competitive enough to extract background from challenging videos.
\begin{figure}
     \centering
     \begin{subfigure}[]{0.15\textwidth}
         \centering
         \includegraphics[width=2.5cm,height=1.8cm]{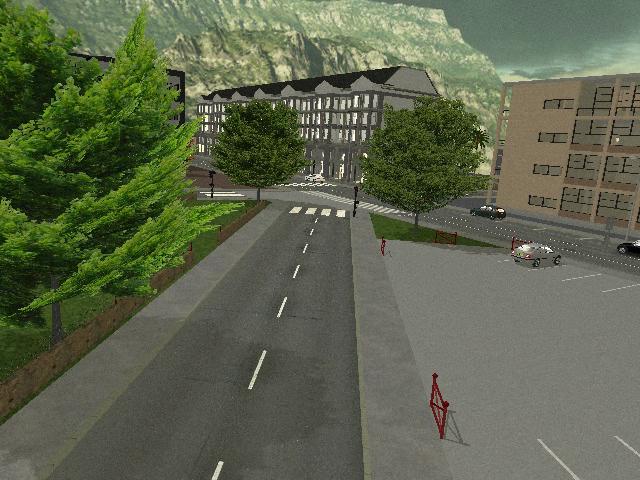}
         \caption*{$511$}
         \label{fig:511}
     \end{subfigure}
     \quad
     \begin{subfigure}[]{0.15\textwidth}
         \centering
         \includegraphics[width=2.5cm,height=1.8cm]{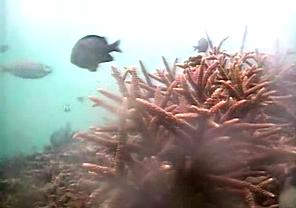}
         \caption*{Blurred}
         \label{fig:Blurred}
     \end{subfigure}
     \quad
     \begin{subfigure}[]{0.15\textwidth}
         \centering
         \includegraphics[width=2.5cm,height=1.8cm]{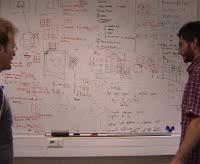}
         \caption*{Board}
         \label{fig:Board}
     \end{subfigure}
      \quad
     \begin{subfigure}[]{0.15\textwidth}
         \centering
         \includegraphics[width=2.5cm,height=1.8cm]{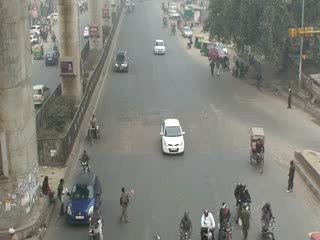}
         \caption*{boulevardJam}
         \label{fig:boulvardJam}
     \end{subfigure}
     \quad

     \begin{subfigure}[]{0.15\textwidth}
         \centering
         \includegraphics[width=2.5cm,height=1.8cm]{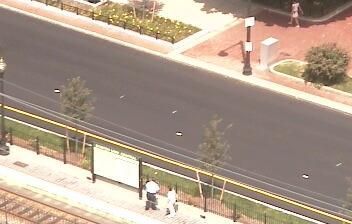}
         \caption*{boulevard}
         \label{fig:boulvardJam}
     \end{subfigure}
     \quad
      \begin{subfigure}[]{0.15\textwidth}
         \centering
         \includegraphics[width=2.5cm,height=1.8cm]{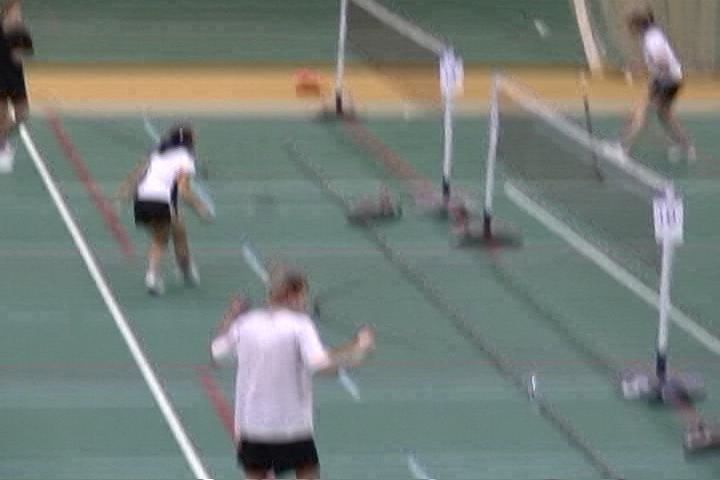}
         \caption*{Badminton}
         \label{fig:Badminton}
     \end{subfigure}
     \quad
      \begin{subfigure}[]{0.15\textwidth}
         \centering
         \includegraphics[width=2.5cm,height=1.8cm]{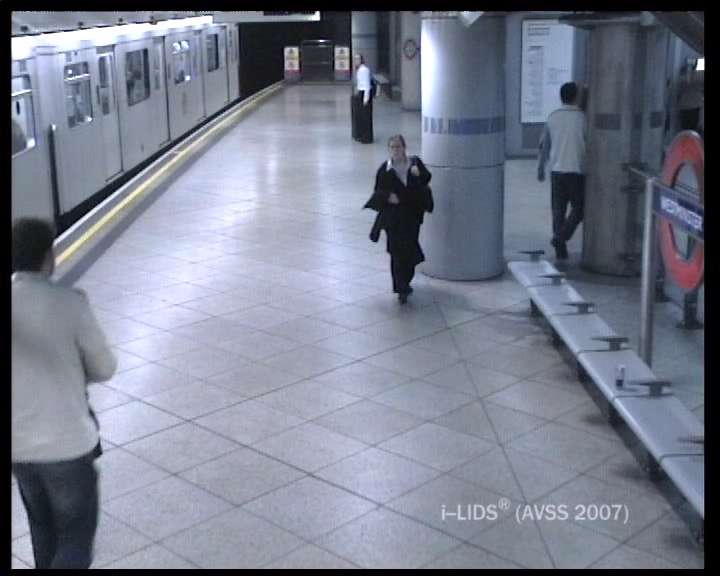}
         \caption*{AVSS2007}
         \label{fig:AVSS2007}
     \end{subfigure}
     \quad
      \begin{subfigure}[]{0.15\textwidth}
         \centering
         \includegraphics[width=2.5cm,height=1.8cm]{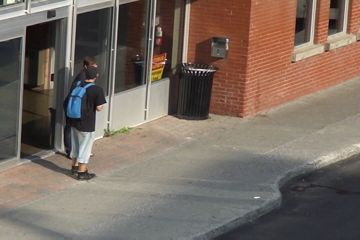}
         \caption*{Bus Station}
         \label{fig:busStation}
     \end{subfigure}
        \caption{The first frame of the eight videos from SBMnet dataset.}
        \label{fig:SBM_input}
\end{figure}

 \begin{table}[htbp]
\scriptsize
\caption{Quantitative quality indexes of DMD method and proposed Q-DMD method on the 8 color videos. }
\centering
\renewcommand\arraystretch{1.7}{
\setlength{\tabcolsep}{0.55mm}{
\begin{tabular}{lllllllllllll}
\hline
\multicolumn{1}{l}{\multirow{2}*{Videos}} &\multicolumn{2}{c}{AGE} & \multicolumn{2}{c}{pEPs}& \multicolumn{2}{c}{pCEPs}& \multicolumn{2}{c}{MSSSIM}& \multicolumn{2}{c}{PSNR}& \multicolumn{2}{c}{CQM}\\
\cline{2-13} 
\multicolumn{1}{l}{} & DMD & Q-DMD & DMD & Q-DMD & DMD & Q-DMD & DMD & Q-DMD & DMD & Q-DMD & DMD & Q-DMD\\
\hline

Basic (511)& 4.0569  &$\mathbf{4.0425}$  &0.0378    & $\mathbf{0.0373}$ & 0.0013&$\mathbf{0.0013 }$ & 0.9780&$\mathbf{0.9783}$ &30.2008&$\mathbf{30.2637}$&31.8512 &$\mathbf{ 32.1079}$ \\

\cline{1-13}

Basic (Blurred)   & 11.3377 &$\mathbf{9.4219}$ & 0.1512 & $\mathbf{0.1033}$ & 0.1233 &$\mathbf{0.2187}$ & 0.9951& 0.9950 & 37.6927&$\mathbf{38.0074}$ &36.9858 &$\mathbf{37.3188}$\\
\cline{1-13}

Clutter (board)  & 33.0994 & $\mathbf{28.5052}$  & 0.7336 & $\mathbf{0.4899}$ &0.6559 & $\mathbf{0.4098}$ &0.5064 & $\mathbf{0.5570}$&16.7050 &16.6674 & 17.7580& $\mathbf{18.1454}$\\
\cline{1-13}

Clutter (boulevardJam) & 4.1226 & $\mathbf{3.6353 }$  & 0.0095 & $\mathbf{0.0095}$ & 0.0029 &0.0031 & 0.9314 & $\mathbf{0.9325}$ & 31.4713 & $\mathbf{31.5366}$ & 32.5999 & $\mathbf{32.6002}$\\
\cline{1-13}

Jitter (boulevard) & 9.4585 & $\mathbf{9.4098}$ & 0.1375   & $\mathbf{0.1372}$   & 0.0277 & $\mathbf{0.0276}$ & 0.9116 & $\mathbf{0.9116}$ & 22.6940 & $\mathbf{22.7186}$& 24.1742 &$\mathbf{24.2013}$\\
\cline{1-13}

Jitter (Badminton) & 5.6188 & $\mathbf{3.9600}$ &  0.0283 &$\mathbf{ 0.0257}$ & 0.0132 &$\mathbf{0.0115}$ &0.9541 & $\mathbf{0.9570}$& 30.2312& $\mathbf{31.3778}$ &31.0495 &$\mathbf{32.2622}$\\
 \cline{1-13}

\makecell[l]{IntermittentMotion \\ (AVSS2007)}  &  17.2716 & 17.3175 &  0.2610  & $\mathbf{0.2589}$  & 0.2065 &$\mathbf{0.2031}$ &0.7127&0.7123 & 19.3044 &$\mathbf{19.3450}$ & 20.2489 &$\mathbf{20.2835}$\\
\cline{1-13}

\makecell[l]{IntermittentMotion \\ (BusStation)} &  6.4834 & $\mathbf{6.3356}$  &  0.0497  & 0.0611  & 0.0255 &0.0328 &0.9512 &0.9505 & 28.3237 &28.0328 & 29.2942&29.1266\\
\cline{1-13}
\label{dmd and qdmd sbm}
\end{tabular}}}
\end{table}

\begin{table}[htbp]
\scriptsize
\caption{Comparison of AGE, pEPs, pCEPs, MSSSIM, PSNR, and CQM between proposed Q-DMD method and other state-of-the-art background initialization methods on SBMnet dataset. Numbers that are bold and underlined, numbers that are bold and dashed, and numbers that are bold represent the highest, second and third highest CQM values on each video, respectively.}
\centering
\renewcommand\arraystretch{1.4}{
\setlength{\tabcolsep}{0.6mm}{
\begin{tabular}{llllllllllll}
\hline
Videos &  & RMR  & \makecell[c]{FC-Flow- \\ Net}   & \makecell[c]{BE-AA- \\ PSA}   & \makecell[c]{Bidirec-\\tional \\ Analysis}    & \makecell[c]{Photo- \\ montage} & NExBI & LaBGen & FSBE  & MSCL & Q-DMD  \\
\hline

\multirow{4}{*}{Basic (511)} & AGE & 5.3709 & 3.9735 & 4.0511  & 4.5214 & 5.79770 & 5.8916 & 4.8294 & 3.7414 & 4.2186 & 4.0425\\
 &  MSSSIM & 0.9457 & 0.9735 & 0.9744 & 0.9705 & 0.9488  & 0.9345  & 0.9475 & 0.9761 & 0.9703 & 0.9783\\ 
 &  PSNR & 26.3268 & 30.8573 & 30.0319 & 28.8396 & 26.6706 & 26.2599 & 27.6577 & 30.5804 & 30.0808 & 30.2637 \\ 
 &  CQM & 28.3708 &\underline{\textbf{32.5541}  }& 31.8292 & 30.7336 & 28.7131 & 28.3762 & 29.5002 &\dashuline {\textbf{32.2388}} & 31.8784 & \textbf{32.1079}\\ 
 
 \hline
 
\multirow{4}{*}{Basic (Blurred)} & AGE & 2.9910 & 2.6962  & 15.2057  & 2.4346 & 2.0214 &  2.5863 & 1.3990  &  3.1953 & 1.8057  & 9.4219 \\
 &  MSSSIM & 0.9699 &  0.9902 & 0.8924  & 0.9924 & 0.9941 & 0.9909  &  0.9975 & 0.9882  & 0.9930  & 0.9596\\ 
 &  PSNR & 30.4749 & 36.3751  &  22.4556 & 37.4609 & 38.2473 & 36.3266  & 41.5779  &  31.8882 & 38.1747  & 25.5840\\ 
 &  CQM & 31.0951 & 36.8199  &  23.3364 & 37.7694 & \dashuline {\textbf{38.5613}} & 36.6845  & \underline {\textbf {41.6541}} & 32.4592  & \textbf{38.5264}  &  26.4187\\ 
 \hline
 
 \multirow{4}{*}{Clutter (board)} & AGE & 7.0139 & 14.1523  &  25.4532 & 8.6680 & 13.4739 &  6.7738 & 8.0208  &  5.5795 & 6.0836  & 28.5052\\
 &  MSSSIM & 0.8337 & 0.8691  & 0.7629  & 0.8957 & 0.5029 & 0.9162  & 0.8491  & 0.9340  & 0.9322  & 0.5570\\ 
 &  PSNR & 28.3130 & 22.1587  &  15.6631 & 22.5686 & 18.8444 &  28.1156 & 27.4114  & 29.7845  & 29.2266  & 16.6674\\ 
 &  CQM & \dashuline {\textbf{29.3061} }& 23.2484  &  16.9305 & 23.7998 & 20.0911 & \textbf {29.0466} & 28.3713  &\underline {\textbf {30.7618}}  & 19.5182  & 18.1454 \\ 
 \hline
 
 \multirow{4}{*}{\makecell[l]{Clutter \\  (boulevardJam)}} & AGE & 4.8947 & 5.0200  &  5.1418 & 7.7770 & 12.1045 & 5.0516  & 8.2239  & 2.3321  &  5.0010 & 3.6353\\
 &  MSSSIM & 0.9282 & 0.8619  & 0.9219  & 0.8585 & 0.7604 & 0.8789  & 0.6851  &  0.9653 &  0.9100 & 0.9325\\ 
 &  PSNR & 29.2511 &  30.3476 &  28.9114 & 24.2706 & 20.9163 & 27.6165  & 22.6515  & 33.8660  &  28.5787 & 31.5366\\ 
 &  CQM & 30.5310 & \textbf {31.4309}  &  30.0986 & 25.4284 & 22.1436 &  28.8454 & 23.9772  & \underline {\textbf{35.0117}}  & 29.8099  & \dashuline{\textbf  {32.6002}}\\ 
 \hline
 
 \multirow{4}{*}{\makecell[c]{Jitter \\ (boulevard)} } & AGE & 13.4511 &  10.6830 &  10.8262 & 10.9028 & 9.7829 & 9.4182  &  10.1888 & 10.1060  &  5.8660 & 9.4098	\\
 &  MSSSIM & 0.8198 & 0.8956  & 0.8821  & 0.8715 & 0.8995 &  0.9076 &  0.8946 & 0.9003  &  0.9699 & 0.9116\\ 
 &  PSNR & 19.5784 & 22.5246  & 21.1393  & 20.2970 &  21.6868 & 22.2455  &  21.4645 & 22.5280  & 26.0077  & 22.7186\\ 
 &  CQM & 21.0043 & \textbf  {24.0208} & 22.5861  & 21.8557 & 23.0513 &  23.7767 & 22.9249  & 23.8107  &\underline {\textbf  {27.1642}} & \dashuline {\textbf   {24.2013}}\\ 
 \hline
 
 \multirow{4}{*}{\makecell[c]{Jitter \\ (Badminton)}  } & AGE & 8.4681 &  5.5368 &  4.3975 & 5.1114 & 4.2924 & 5.2289  & 2.2670  & 6.5668  &  2.4174 & 3.9600\\
 &  MSSSIM & 0.7365 & 0.9367  & 0.9204  & 0.8954 & 0.9237 &  0.8726 & 0.9805  & 0.8636  &  0.9729 & 0.9570\\ 
 &  PSNR & 23.8541 & 29.9097  & 29.1352  & 27.2333 & 29.6868 &  26.7733 & 34.6482  & 27.3765  &  33.8911 & 31.3778\\ 
 &  CQM & 24.7652 & 30.7442  & 29.9490  & 28.1560 & 30.4911 & 27.7054  &\underline {\textbf {35.2688} } &  28.2185 & \dashuline {\textbf  {34.5949} } & \textbf {32.2622}\\ 
 \hline
 
 \multirow{4}{*}{\makecell[l]{Intermittent- \\Motion\\ (AVSS2007)}} & AGE & 9.2767 &  11.6751 &  20.6172 & 11.9126 & 12.0167 &  12.3242 & 8.3062  & 11.5900  & 7.5256  & 17.3175\\
 &  MSSSIM & 0.9094 &  0.8726 & 0.7929  & 0.8198 & 0.8400 & 0.8799  &  0.9050 &  0.8830 &  0.9294 & 0.7123\\ 
 &  PSNR & 20.3096 & 20.7442  & 16.4960  & 19.6485 & 19.2860 & 21.1518  & 21.4577  & 20.1106  & 22.3138  & 19.3450\\ 
 &  CQM & 21.3404 & 21.7565  &  17.5546 & 20.7738 & 20.2173 &\textbf {22.0076}  &\dashuline {\textbf {22.3158} } & 21.2110  &\underline {\textbf {23.0990} } &  20.2835\\ 
 \hline

\multirow{4}{*}{\makecell[l]{Intermittent- \\Motion \\(Bus Station)} } & AGE & 3.1366 & 4.3513  & 4.5206  & 4.3423 & 6.5309 & 3.0622  & 7.0296  & 4.3997  & 3.4057  & 6.3356\\
 &  MSSSIM & 0.9631 & 0.9622  & 0.9621  & 0.9651 & 0.8872 & 0.9815  & 0.8889  & 0.9847  & 0.9821  & 0.9505\\ 
 &  PSNR & 30.3210 & 31.1049  & 30.0286  & 28.0407 & 21.8651 & 35.2212  & 22.0988  & 33.1076  & 34.2369  & 28.0328\\ 
 &  CQM & 31.4297 & 31.7573  & 30.9833  & 29.0178 & 22.8979 &\underline {\textbf {35.7016}}  & 23.0664  &\textbf {33.7125}  & \dashuline {\textbf{34.8402}}  &  29.1266
\\ 
 \hline
 \label{qdmd and other sbm}
\end{tabular}}}
\end{table}

\begin{figure}
     \centering
     \begin{subfigure}{0.15\textwidth}
         \centering
         \includegraphics[width=2.5cm,height=1.8cm]{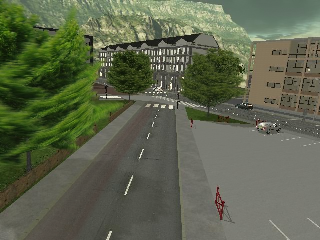}
         \caption*{$511$}
         \label{fig:511}
     \end{subfigure}
     \quad
     \begin{subfigure}{0.15\textwidth}
         \centering
         \includegraphics[width=2.5cm,height=1.8cm]{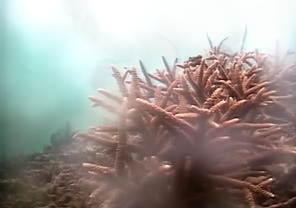}
         \caption*{Blurred}
         \label{fig:Blurred}
     \end{subfigure}
     \quad
     \begin{subfigure}{0.15\textwidth}
         \centering
         \includegraphics[width=2.5cm,height=1.8cm]{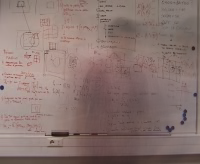}
         \caption*{Board}
         \label{fig:Board}
     \end{subfigure}
      \quad
     \begin{subfigure}{0.15\textwidth}
         \centering
         \includegraphics[width=2.5cm,height=1.8cm]{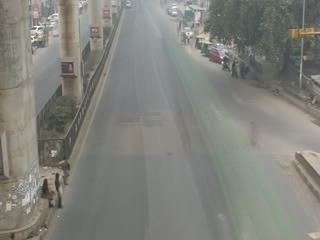}
         \caption*{boulevardJam}
         \label{fig:boulvardJam}
     \end{subfigure}
     \quad

     \begin{subfigure}{0.15\textwidth}
         \centering
         \includegraphics[width=2.5cm,height=1.8cm]{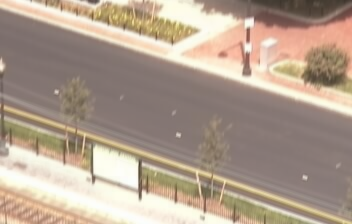}
         \caption*{boulevard}
         \label{fig:boulvardJam}
     \end{subfigure}
     \quad
      \begin{subfigure}{0.15\textwidth}
         \centering
         \includegraphics[width=2.5cm,height=1.8cm]{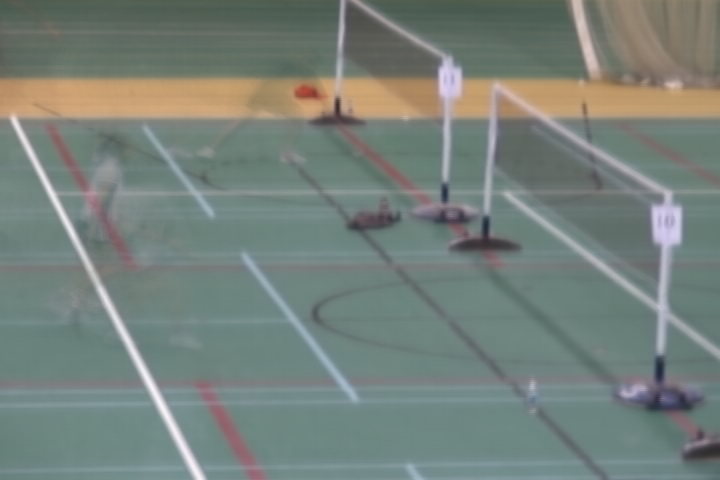}
         \caption*{Badminton}
         \label{fig:Badminton}
     \end{subfigure}
     \quad
      \begin{subfigure}{0.15\textwidth}
         \centering
         \includegraphics[width=2.5cm,height=1.8cm]{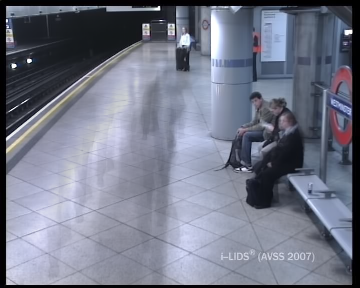}
         \caption*{AVSS2007}
         \label{fig:AVSS2007}
     \end{subfigure}
     \quad
      \begin{subfigure}{0.15\textwidth}
         \centering
         \includegraphics[width=2.5cm,height=1.8cm]{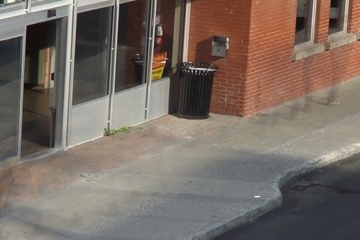}
         \caption*{Bus Station}
         \label{fig:busStation}
     \end{subfigure}
        \caption{Results of the generated backgrounds on SBMnet dataset.}
        \label{fig:SBM}
\end{figure}

\textbf{Experiment on the SBI dataset:} We conducted additional experiments on the SBI dataset to evaluate our background modeling results by comparing them with the background initialization results obtained by DMD method. The SBI dataset also contains a large amount of data extracted from original publicly available sequences, which are frequently used in the literature to evaluate background initialization algorithms \cite{maddalena2015towards}. We evaluated our approach on 9 sequences of the SBI dataset. The remaining videos, "Cavignal" and "CAVIAR1" are objects with intermittent object motion which are not defined as moving objects, and "Snellen" and "PeopleAndFoliage" are those videos with very short background exposure duration, and "Toscana" has only five frames and is not suitable for Q-DMD method. The background frames that did not change over time were trimmed and among the nine videos, except those with less than 200 frames, we extracted a set of 200 frames of continuous videos for experiments to reduce the computing time.  We calculated the six metrics (AGE, pEPs, pCEPs, MSSSIM, PSNR and CQM) suggested by the SBI dataset to measure the reconstructed background models, and the results are shown in Table \ref{SBI:TABLE}. Generated backgrounds with Q-DMD method on SBI dataset are shown in Figure \ref{fig:SBI}. 

\indent
The Q-DMD method only considers the part of the image sequence that does not change with time as the background model. When this condition is met, the Q-DMD can achieve better performance. For example, for the three sequences, "HighwayI", "HighwayII" and "IBMtest2", the foregrounds (people or vehicles) do not remain stationary anywhere in the scene throughout the sequence. Therefore, for these color video sequences, Q-DMD can generate a superb background model by eliminating the the foreground part of continuous motion. However, when this condition is not satisfied, it will affect the results of Q-DMD. For example, on the " Board" sequence, the background models reconstructed by Q-DMD and DMD had human shadow. This is because there are two men in the video sequence who occupy a large proportion of the background in the whole video sequence, and the man standing on the right side of the sequence has been rotating for some time in the sequence, which leads to the background covered by the foregrounds for a long time. As can be seen from Figure \ref{fig:SBI}, for the three color videos, "Board", "HighwayI" and "HighwayII", the backgrounds reconstructed by DMD method had obvious different color intensity compared to the ground truth, however the background models generated by Q-DMD method do not have this problem. Therefore, we have reason to believe that this is mainly due to the advantage of quaternion in representing color pixel values. As shown in Table 1, for most videos on the SBI dataset, our AGE, pEPs, and pCEPs were lower than those of the DMD method, which indicates that it has the lower pixel-wise difference between the reconstructed background model and the ground truth model. PSNR, MS-SSIM, and CQM also show that the Q-DMD method has more obvious advantages over DMD method. \\

\begin{table}[htbp]
\scriptsize
\caption{Evaluation results (SBI dataset)}
\label{SBI:TABLE}
\centering
\renewcommand\arraystretch{1.5}{
\setlength{\tabcolsep}{0.6mm}{
\begin{tabular}{lllllllllllll}
\hline
\multicolumn{1}{l}{\multirow{2}*{Videos}} &\multicolumn{2}{c}{AGE} & \multicolumn{2}{c}{pEPs\%}& \multicolumn{2}{c}{pCEPs\%}& \multicolumn{2}{c}{MSSSIM}& \multicolumn{2}{c}{PSNR}& \multicolumn{2}{c}{CQM}\\
\cline{2-13}
\multicolumn{1}{l}{} & DMD & Q-DMD & DMD & Q-DMD & DMD & Q-DMD & DMD & Q-DMD & DMD & Q-DMD & DMD & Q-DMD \\
\hline
\text{Board} & 28.3809 & $\mathbf{24.5430}$  & 67.8445 & $\mathbf{43.4299}$ &59.6707 & $\mathbf{36.6799}$ &0.5406 & $\mathbf{0.5967}$&18.0398 &17.9744 & 17.7580& $\mathbf{18.1454}$\\
\cline{1-13}
$\text{Candelam1}\_ \text{m1}.10$ & 3.1744  &$\mathbf{3.1249}$  &1.3780    & 1.4076 & 0.7339&0.7576  & 0.9653&$\mathbf{0.9654}$ & 31.6111&$\mathbf{31.6436}$&30.8939 &30.8877 \\
\cline{1-13}
CAVIAR2 & 1.0730 & $\mathbf{1.0687}$  &  0.0041 & $\mathbf{0.0010 }$ &0.0000 &$\mathbf{0.0000 }$ & 0.9988&$\mathbf{ 0.9988}$ & 43.3340&$\mathbf{43.7556}$ &42.5350 &$\mathbf{42.9473}$\\
\cline{1-13}
Foliage& 22.2094 & $\mathbf{20.8727}$ &43.2361   &48.8854   & 31.1806 & 34.1458&0.7975&0.7644 &19.2886 & $\mathbf{20.0956}$&19.2072 &$\mathbf{19.6507}$\\
\cline{1-13}
HallAndMonitor&  3.7244 & 3.8173  &  2.9616  &$\mathbf{2.9096  }$  &1.6359 &$\mathbf{ 1.6098}$ &0.9561  &$\mathbf{0.9561}$ &29.9361 &$\mathbf{29.9827 }$ & 29.9230&$\mathbf{29.9749}$\\
\cline{1-13}
HighwayI& 43.2153 & $\mathbf{5.4775}$ &  99.9688 &$\mathbf{ 0.1198}$ & 99.8958 &$\mathbf{0.0039}$ &0.9015 & $\mathbf{0.9649}$& 15.2265& $\mathbf{31.3369}$ &15.1443 &$\mathbf{31.2429}$\\
\cline{1-13}
HighwayII&3.1864  &$\mathbf{2.7014}$  &0.3164 &$\mathbf{0.3125 }$  &0.0000& 0.0013 & 0.9911 & $\mathbf{ 0.9917}$&35.1937& $\mathbf{36.1269}$&34.8043&$\mathbf{35.6207}$\\
\cline{1-13}
HumanBody2& 7.1215   &$\mathbf{7.0930}$  &6.6536   &6.7174   & 3.9896 &4.0156 &0.9511  &$\mathbf{0.9515 }$ &25.9354 & $\mathbf{25.9580 }$&25.5177 &$\mathbf{25.4730}$\\
\cline{1-13}
IBMtest2& 4.9182 &  $\mathbf{4.8208}$ & 2.7122 &$\mathbf{2.2513}$  & 0.8477& $\mathbf{ 0.6133 }$&0.9826  & $\mathbf{0.9827}$& 30.7902&$\mathbf{30.9653 } $& 29.9444&$\mathbf{30.0359}$\\
\cline{1-13}
\hline
\end{tabular}}}
\end{table}

\begin{figure}
\centering
 \begin{subfigure}{0.1\textwidth}
         \centering
         \setlength{\abovecaptionskip}{0.4cm}
         \caption*{Board}
         \label{fig:boulvardJam}
         \includegraphics[width=1.61cm,height=1.5cm]{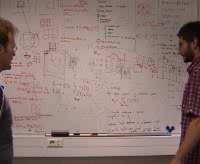}
     \end{subfigure} 
     \hfill
\begin{subfigure}{0.1\textwidth}
         \centering
        \setlength{\abovecaptionskip}{0cm}
         \caption*{$\text{Candelam1-}\\ \_ \text{m1}.10$}
         \label{fig:boulvardJam}
         \includegraphics[width=1.61cm,height=1.5cm]{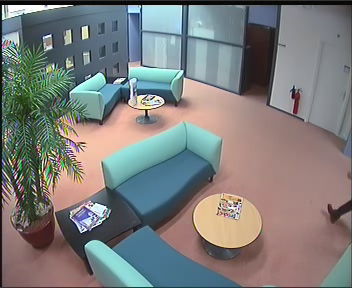}
     \end{subfigure} 
     \hfill
\begin{subfigure}{0.1\textwidth}
         \centering
         \setlength{\abovecaptionskip}{0.4cm}
         \caption*{CAVIAR2}
         \label{fig:boulvardJam}
         \includegraphics[width=1.61cm,height=1.5cm]{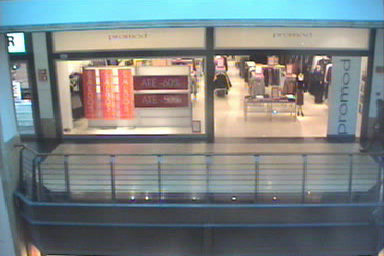}
     \end{subfigure} 
     \hfill
     \begin{subfigure}{0.1\textwidth}
         \centering
         \setlength{\abovecaptionskip}{0.4cm}
         \caption*{Foliage}
         \label{fig:boulvardJam}
         \includegraphics[width=1.61cm,height=1.5cm]{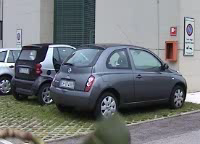}
     \end{subfigure} 
     \hfill
     \begin{subfigure}{0.1\textwidth}
         \centering
         \setlength{\abovecaptionskip}{0.0cm}
         \caption*{HallAndMonitor}
         \label{fig:boulvardJam}
         \includegraphics[width=1.61cm,height=1.5cm]{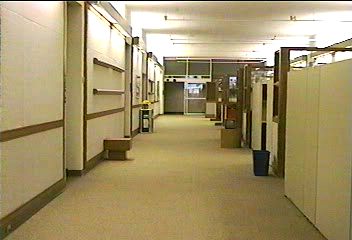}
     \end{subfigure} 
     \hfill
     \begin{subfigure}{0.1\textwidth}
         \centering
         \setlength{\abovecaptionskip}{0.4cm}
         \caption*{HighwayI}
         \label{fig:boulvardJam}
        \includegraphics[width=1.61cm,height=1.5cm]{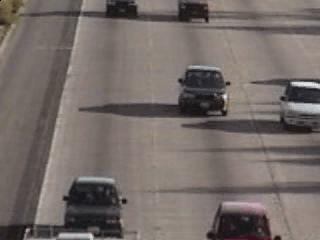}
     \end{subfigure} 
     \hfill
     \begin{subfigure}{0.1\textwidth}
         \centering
         \setlength{\abovecaptionskip}{0.4cm}
         \caption*{HighwayII}
         \label{fig:boulvardJam}
         \includegraphics[width=1.61cm,height=1.5cm]{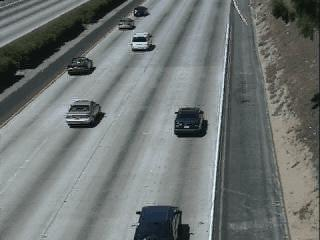}
     \end{subfigure} 
     \hfill
     \begin{subfigure}{0.1\textwidth}
         \centering
         \setlength{\abovecaptionskip}{0.0cm}
         \caption*{HumanBody2}
         \label{fig:boulvardJam}
        \includegraphics[width=1.61cm,height=1.5cm]{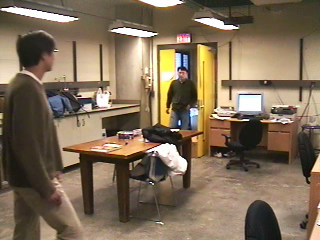}
     \end{subfigure}
     \hfill
     \begin{subfigure}{0.1\textwidth}
         \centering
         \setlength{\abovecaptionskip}{0.4cm}
         \caption*{IBMtest2}
         \label{fig:boulvardJam}
        \includegraphics[width=1.61cm,height=1.5cm]{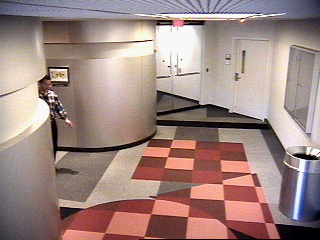}
     \end{subfigure}\hfill\\
     
      \begin{subfigure}{0.1\textwidth}
         \centering
         \includegraphics[width=1.61cm,height=1.5cm]{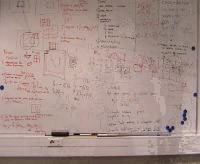}
     \end{subfigure} 
     \hfill
\begin{subfigure}{0.1\textwidth}
         \centering
         \includegraphics[width=1.61cm,height=1.5cm]{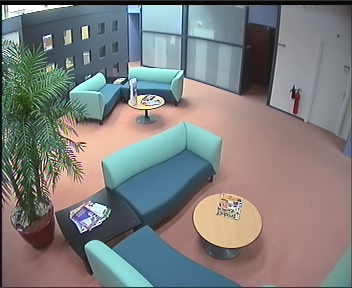}
     \end{subfigure} 
     \hfill
\begin{subfigure}{0.1\textwidth}
         \centering
         \includegraphics[width=1.61cm,height=1.5cm]{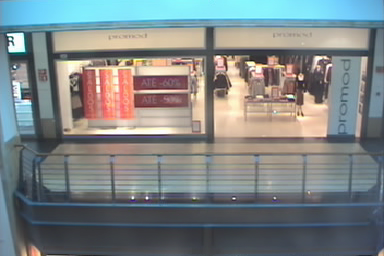}
     \end{subfigure} 
     \hfill
     \begin{subfigure}{0.1\textwidth}
         \centering
         \includegraphics[width=1.61cm,height=1.5cm]{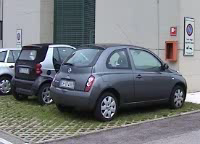}
     \end{subfigure} 
     \hfill
     \begin{subfigure}{0.1\textwidth}
         \centering
         \includegraphics[width=1.61cm,height=1.5cm]{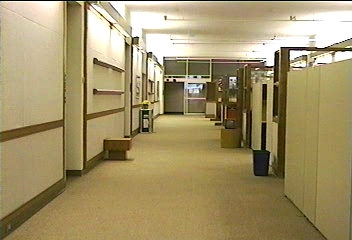}
     \end{subfigure} 
     \hfill
     \begin{subfigure}{0.1\textwidth}
         \centering
        \includegraphics[width=1.61cm,height=1.5cm]{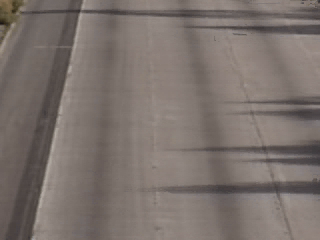}
     \end{subfigure} 
     \hfill
     \begin{subfigure}{0.1\textwidth}
         \centering
         \includegraphics[width=1.61cm,height=1.5cm]{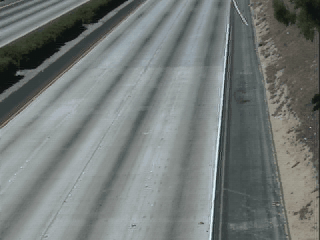}
     \end{subfigure} 
     \hfill
     \begin{subfigure}{0.1\textwidth}
         \centering
        \includegraphics[width=1.61cm,height=1.5cm]{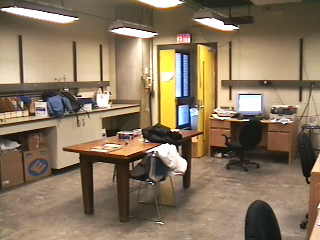}
     \end{subfigure}
     \hfill
     \begin{subfigure}{0.1\textwidth}
         \centering
        \includegraphics[width=1.61cm,height=1.5cm]{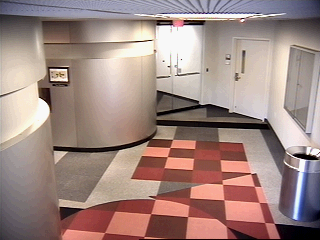}
     \end{subfigure}\hfill\\
     
       \begin{subfigure}{0.1\textwidth}
         \centering
         \includegraphics[width=1.61cm,height=1.5cm]{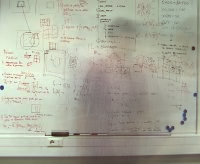}
     \end{subfigure} 
     \hfill
\begin{subfigure}{0.1\textwidth}
         \centering
         \includegraphics[width=1.61cm,height=1.5cm]{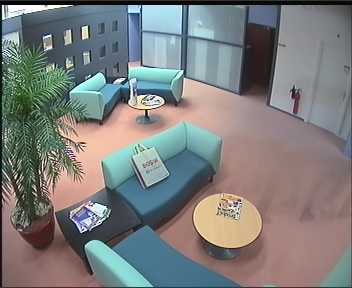}
     \end{subfigure} 
     \hfill
\begin{subfigure}{0.1\textwidth}
         \centering
         \includegraphics[width=1.61cm,height=1.5cm]{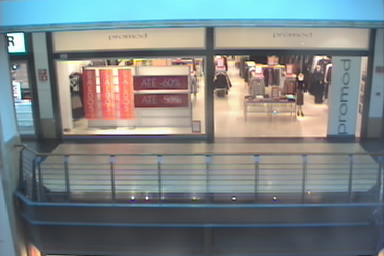}
     \end{subfigure} 
     \hfill
     \begin{subfigure}{0.1\textwidth}
         \centering
         \includegraphics[width=1.61cm,height=1.5cm]{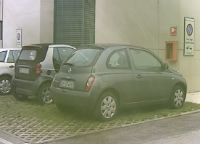}
     \end{subfigure} 
     \hfill
     \begin{subfigure}{0.1\textwidth}
         \centering
         \includegraphics[width=1.61cm,height=1.5cm]{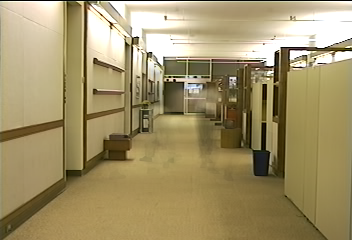}
     \end{subfigure} 
     \hfill
     \begin{subfigure}{0.1\textwidth}
         \centering
        \includegraphics[width=1.61cm,height=1.5cm]{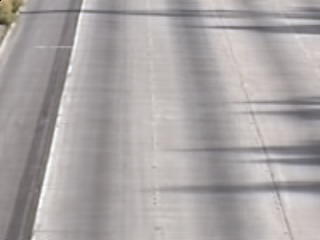}
     \end{subfigure} 
     \hfill
     \begin{subfigure}{0.1\textwidth}
         \centering
         \includegraphics[width=1.61cm,height=1.5cm]{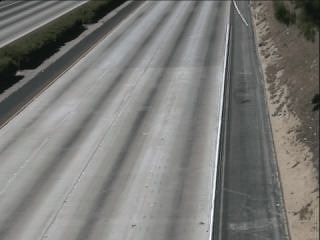}
     \end{subfigure} 
     \hfill
     \begin{subfigure}{0.1\textwidth}
         \centering
        \includegraphics[width=1.61cm,height=1.5cm]{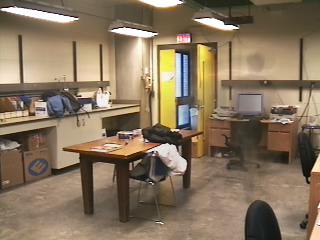}
     \end{subfigure}
     \hfill
     \begin{subfigure}{0.1\textwidth}
         \centering
        \includegraphics[width=1.61cm,height=1.5cm]{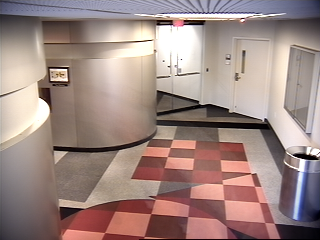}
     \end{subfigure}\hfill\\
     
       \begin{subfigure}{0.1\textwidth}
         \centering
         \includegraphics[width=1.61cm,height=1.5cm]{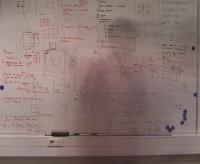}
     \end{subfigure} 
     \hfill
\begin{subfigure}{0.1\textwidth}
         \centering
         \includegraphics[width=1.61cm,height=1.5cm]{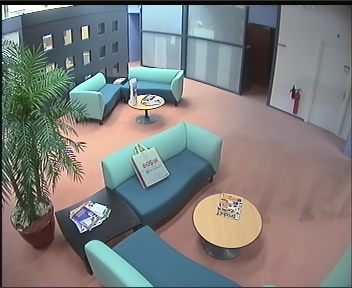}
     \end{subfigure} 
     \hfill
\begin{subfigure}{0.1\textwidth}
         \centering
         \includegraphics[width=1.61cm,height=1.5cm]{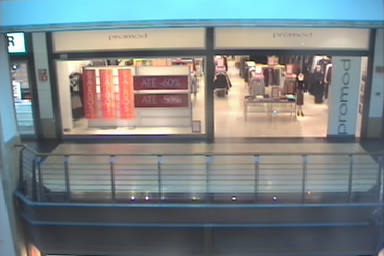}
     \end{subfigure} 
     \hfill
     \begin{subfigure}{0.1\textwidth}
         \centering
         \includegraphics[width=1.61cm,height=1.5cm]{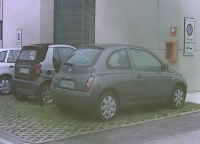}
     \end{subfigure} 
     \hfill
     \begin{subfigure}{0.1\textwidth}
         \centering
         \includegraphics[width=1.61cm,height=1.5cm]{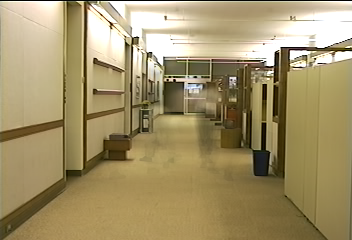}
     \end{subfigure} 
     \hfill
     \begin{subfigure}{0.1\textwidth}
         \centering
        \includegraphics[width=1.61cm,height=1.5cm]{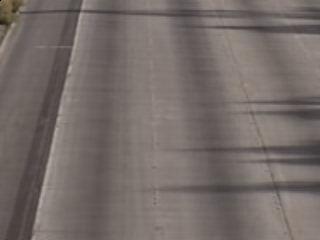}
     \end{subfigure} 
     \hfill
     \begin{subfigure}{0.1\textwidth}
         \centering
         \includegraphics[width=1.61cm,height=1.5cm]{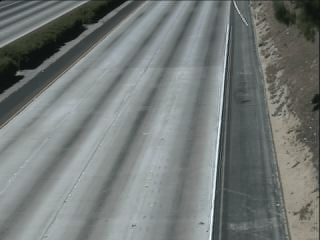}
     \end{subfigure} 
     \hfill
     \begin{subfigure}{0.1\textwidth}
         \centering
        \includegraphics[width=1.61cm,height=1.5cm]{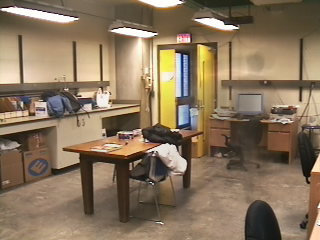}
     \end{subfigure}
     \hfill
     \begin{subfigure}{0.1\textwidth}
         \centering
        \includegraphics[width=1.61cm,height=1.5cm]{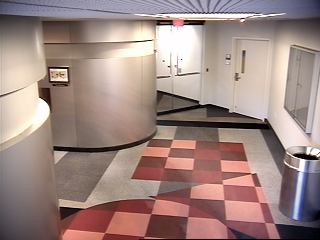}
     \end{subfigure}\hfill\\
\caption{\label{fig:SBI} First-row: the first frame of the nine videos from the SBI dataset; second-row: Ground truth of the background; Third-row: generated backgrounds with DMD method; Last-row: generated backgrounds using Q-DMD method.}
\end{figure}

\section{Conclusions}
In this paper, we propose a quaternion-based DMD (Q-DMD) method for color video background modeling using quaternion matrix analysis. Quaternion representation treats color pixels as vector units rather than scalars, naturally processes the coupling between color channels, and fully retains the color information of the color image or video. A high-order real tensor can be used to present a color video, however, the color structure will be destroyed in the process of matricization (e.g., mode-k unfolding). Using the standard eigenvalue of quaternion, we establish the spectral decomposition of the quaternion matrix, and then extend DMD to quaternion system, i.e., Q-DMD. The results demonstrate that compared with DMD, our method shows advantages in reconstructing color video background model, and compared with several state-of-art methods, the proposed method still has competitive performance (w.r.t., CQM).

Note that the proposed method can reconstruct a better background model for videos that meet the requirements (i.e., Q-DMD method only considers the parts of the color videos that do not change with time as the background model). On the contrary, for videos that do not meet the conditions, the effect of the reconstructed background model will be reduced. Therefore, we consider combining Q-DMD with other low rank sparse algorithms to make Q-DMD more applicable in the future work.

\normalem
\bibliographystyle{unsrt}
\bibliography{sample}

\end{document}